\documentclass[11pt]{article}

\usepackage{amsmath,amsbsy,amsfonts,amssymb,amsthm,fullpage,units,bm}
\usepackage{graphicx}
\usepackage{subcaption}
\usepackage{algorithm,algorithmic,mathtools}
\usepackage{color,cases}
\usepackage{tikz}
\usepackage{multirow}
\usepackage{hyperref}
\usepackage{authblk}

\hypersetup{
colorlinks = true,
citecolor = {blue},
citebordercolor = {white}
}

\usepackage[top=1.1in, bottom=1.1in, left=1.1in, right=1.1in]{geometry}

\newtheorem{theorem}{Theorem}
\newtheorem{remark}{Remark}
\newtheorem{proposition}[theorem]{Proposition}
\newtheorem{lemma}{Lemma}

\theoremstyle{definition}
\newtheorem{definition}{Definition}

\DeclareMathOperator{\sgn}{sign}
\DeclareMathOperator{\rad}{Rad}

\newcommand{\bR}{\mathbb{R}}

\newcommand{\RR}{\mathbb{R}}
\newcommand{\EE}{\mathbb{E}}
\renewcommand{\SS}{\mathbb{S}}

\newcommand{\PP}{\mathbb{P}}

\newcommand{\cH}{\mathcal{H}}
\newcommand{\cP}{\mathcal{P}}
\newcommand{\cF}{\mathcal{F}}

\newcommand{\cO}{\mathcal{O}}
\newcommand{\cD}{\mathcal{D}}
\newcommand{\cB}{\mathcal{B}}
\newcommand{\cE}{\mathcal{E}}

\newcommand{\cA}{\mathcal{A}}

\newcommand{\bx}{\mathbf{x}}
\newcommand{\bX}{\mathbf{X}}
\newcommand{\be}{\mathbf{e}}

\newcommand{\bU}{\mathbf{U}}
\newcommand{\bV}{\mathbf{V}}
\newcommand{\bA}{\mathbf{A}}

\newcommand{\bz}{\bm{z}}
\newcommand{\bw}{\bm{w}}
\newcommand{\bW}{\bm{W}}
\newcommand{\bv}{\bm{v}}

\newcommand{\bb}{\bm{b}}

\newcommand{\balpha}{\bm{\alpha}}

\newcommand{\by}{\mathbf y}

\newcommand{\lip}{\textit{Lip}_{\{\rho_t\}}}
\newcommand{\Lip}{\textit{Lip}}

\title{The Barron Space and the Flow-induced Function Spaces\\
for Neural Network Models}

\author[1,2,3]{Weinan E \thanks{weinan@math.princeton.edu}}
\author[2]{Chao Ma \thanks{cham@princeton.edu}}
\author[2]{Lei Wu \thanks{leiwu@princeton.edu}}

\affil[1]{Department of Mathematics, Princeton University}
\affil[2]{Program in Applied and Computational Mathematics, Princeton University}
\affil[3]{Beijing Institute of Big Data Research}

\begin{document}
\maketitle

\begin{abstract}
One of the key issues in the analysis of machine learning models
is to identify the appropriate function space {and norm} for the model.
{This is the set of functions endowed with a quantity which can control the 
approximation and estimation errors by a particular machine learning model}.
In this paper, we address this issue for two representative neural network models:  the two-layer networks and  the residual neural networks.
We define the Barron space and show that it is the right space for two-layer 
neural network models in the sense that optimal direct and inverse approximation
theorems hold for functions in the Barron space.
For residual neural network models,  we construct the so-called flow-induced function space, and prove direct and inverse approximation theorems for this space.
In addition, we show that the Rademacher complexity  
 for bounded sets under these norms has the optimal upper bounds.
\end{abstract}



\section{Introduction}

The task of supervised learning is to approximate a function using a given set of data.
This type of problem has been the subject of classical numerical analysis and 
approximation theory for a long time. 
The theory of splines and the theory of finite element methods
are very successful examples of such classical results \cite{devore1993constructive, ciarlet2002finite},
both are concerned with approximating functions using piecewise polynomials.
In these theories, one starts from a function in a particular function space, 
say a Sobolev or Besov space, and proceeds to derive
optimal error estimates for this function.
The optimal error estimates depend on the {function norm}, and the regularity encoded in the 
function space as well as the approximation scheme.
They are the most important pieces of information for understanding the
underlying approximation scheme. 
{When discussing a particular function space, the associated norm is as crucial as the set of functions it contains.}

Identifying the right function space that one should use is the most crucial step in this analysis.
Sobolev/Besov type spaces are good function spaces for these classical theories since:
\begin{enumerate}
\item One can prove direct and inverse approximation theorems for these spaces.
Roughly speaking, a function can be approximated by piecewise polynomials with certain convergence rate if and only if the function is in certain Sobolev/Besov space.
\item The functions we are interested in, e.g. solutions of partial differential
equations (PDEs), are in these spaces. This is at the heart of the regularity theory for PDEs.
\end{enumerate}
However, these spaces are tied with the piecewise polynomial basis used in the approximation
scheme. These approximation schemes suffer from the curse of dimensionality, i.e. the number of 
parameters needed to achieve certain level of accuracy grows exponentially with dimension.
Consequently, Sobolev/Besov type spaces are not the right function spaces for studying
machine learning models that can potentially address the curse of dimensionality problem.

Another inspiration for this paper comes from kernel methods. It is well-known
that the right function space associated with a kernel method is the  corresponding 
reproducing kernel Hilbert space (RKHS)~\cite{aronszajn1950theory}.
RKHS and kernel methods provide one of the first examples for which dimension-independent
error estimates can be established.

The main purpose of this paper is to construct and identify
the analog of these spaces for two-layer and residual neural network models.
For two-layer neural network models, we show that the right
function space is the so-called ``Barron space''.
Roughly speaking, a function belongs to the Barron space if and only if 
it can be approximated by ``well-behaved'' two-layer neural networks, 
{and the approximation error is controlled by the norm of the Barron space.}
The analog of the Barron space for deep residual neural networks is the 
 ``flow-induced function space''  that we construct in the second part of this paper. 
{With the ``flow-induced norms'',} we will establish direct and inverse approximation theorems for these spaces as well as the optimal Rademacher complexity estimates.

One important difference between approximation theory in low and high dimensions is that
in high dimensions, the best error rate (or order of convergence) that one can hope for is the Monte Carlo error rate.
Therefore using the error rate as an indicator to distinguish the quality of different approximation schemes or machine
learning models is not a good option. 
The function spaces or the associated norms seem to be a better alternative.
We take the viewpoint that a function space is defined by its approximation property using a particular
approximation scheme. In this sense, Sobolev/Besov spaces are the result when we consider approximation
by piecewise polynomials or wavelets. Barron space is the analog when we consider approximation by two-layer neural
networks and the flow-induced function space is the analog when we consider approximation by deep residual networks.
The norms that are associated with these new spaces may seem a bit unusual at a first sight, but they arise naturally
in the approximation process, as we will see from the  direct and inverse approximation theorems presented below. 

{It should be stressed that the terminologies ``space'' and ``norm'' in this paper are used in a loose way. For example, flow-induced norms are a family of quantities that control the approximation error. We do not take effort to investigate whether it is a real norm.}

Although this work was motivated by the problem of understanding approximation theory for neural network models
in machine learning,  we believe that it should have an implication for high dimensional analysis in general.
One natural follow-up question is whether one can show that solutions to high dimensional partial differential equations
(PDE) belong to the function spaces  introduced here.  At least for linear parabolic PDEs,
the work in \cite{jentzen2018proof} suggests that some close analog of the flow-induced spaces should serve the purpose.

In Section \ref{sec: barron}, we introduce the Barron space for two-layer neural networks. Although not all the results in this section are new (some have appeared in various forms in \cite{klusowski2016risk, bach2017breaking, wu2019priori}), they are useful for illustrating our angle
of attack and they are also useful for the work in Section \ref{sec: cfs} where we introduce the flow-induced function
space for residual networks.

{\bf Notations}: { Let $\SS^d = \{\bw\in\RR^{d+1}: \|\bw\|_1=1\}$. We define $\hat{\bw} = \frac{\bw}{\|\bw\|_1}$ if $\bw\neq 0$ otherwise $\hat{\bw}=0$. } For simplicity, we fix the domain of interest to be $X=[0,1]^d$. 
We denote by $\bx \in X$  the input variable, { and let $\tilde{\bx}=(\bx^T,1)^T$.}
We sometimes abuse notation and use $f(\bx)$ (or some other analogs) to denote the function $f$ in order to signify the independent variable under consideration.
We use $\|f\|$ to denote the $L_2$ norm of function $f$ defined by
\begin{equation}
    \|f\|=\left(\int_X |f(\bx)|^2 \mu(d\bx)\right)^{\frac{1}{2}}, \nonumber
\end{equation}
where $\mu(\bx)$ is a probability distribution on $X$. We do not specify $\mu$ in this paper. 

One important point for working in high
dimension is the dependence of the constants on the dimension.
We will use $C$ to denote constants that are independent of the dimension. 

{In Section~\ref{sec: cfs}, the absolute values and powers of matrices and vectors ($|\cdot|$ and $(\cdot)^p$) are understood as being element-wise. 
The multiplication of two matrices is regular matrix multiplication.}

\section{The Barron space}
\label{sec: barron}

{ In this section we define the Barron space and study its properties. The proofs of theorems are postponed to the end of the section.}

\subsection{Definition of the Barron space}
We will consider
functions $f: X \mapsto \RR$  that admit the following representation
\begin{equation} 
f(\bx) = \int_{\Omega} a \sigma (\bb^T  \bx + c) \rho (da, d\bb, dc), \quad \bx \in X
\label{integral1}
\end{equation}
where $\Omega = \RR^1 \times \RR^{d} \times \RR^1$,
$\rho$ is a probability distribution on ($\Omega$, $\Sigma_{\Omega}$), with $\Sigma_{\Omega}$ being a Borel $\sigma$-algebra on $\Omega$, { and $\sigma(x)=\max\{x,0\}$ is the ReLU activation function.}
This representation can be considered as the
continuum analog of two-layer neural networks:
\begin{equation}\label{eqn: two-layer-net}
f_m(\bx;\Theta):=\frac 1 m \sum_{j=1}^m  a_j \sigma (\bb_j ^T  \bx + c_j), \nonumber
\end{equation}
where $\Theta=(a_1,\bb_1,c_1,\dots,a_m,\bb_m,c_m)$ denotes all the parameters.
It should be noted that in general, the $\rho$'s 
for which \eqref{integral1} holds are not unique.

To get some intuition about the representation \eqref{integral1},
we write the Fourier representation of a function $f$ as:
\begin{equation}\label{eq:fourier}
 f(\bx) = \int_{\RR^d} \hat{f}(\omega) \cos(\omega^T \bx) d \omega = \int_{\RR^1 \times \RR^d} a \cos(\omega^T \bx) \rho(da, d\omega),
\end{equation}
\begin{equation*}
    \rho(da, d\omega) = \delta(a - \hat{f}(\omega)) da d \omega.
\end{equation*}
This can be thought of as the analog of \eqref{integral1} 
for the case when $\sigma(z) = \cos(z)$ except for the
fact that the $\rho$ defined in~\eqref{eq:fourier} is not normalizable.

For functions that admit the representation \eqref{integral1}, we define its Barron norm:
\begin{equation}\label{eqn: barron-norm-def}
\|f \|_{\mathcal{B}_p} = \inf_{\rho} \big( \EE_\rho [|a|^p(\|\bb\|_1 +|c|)^p] \big)^{1/p}, \quad
 1 \le p \le +\infty.
\end{equation} 
Here the infimum is taken over all $\rho$ for which \eqref{integral1} holds for all $\bx \in X$, { and when $p=\infty$ the norm~\eqref{eqn: barron-norm-def} becomes
$$\inf_{\rho} \max_{(a,\bb,c)\in \mbox{supp}(\rho)} |a|(\|\bb\|_1+|c|).$$
}
Barron spaces $\cB_p$ are defined as {the set of continuous 
functions that can be represented by~\eqref{integral1} 
with finite Barron norm.}
We name these spaces after Barron to honor his contribution to the
mathematical analysis of two-layer neural networks, in particular the work in
\cite{barron1993universal,barron1994approximation,klusowski2016risk}.

\begin{remark}
It should be noted that the Barron norm defined here is different from the
spectral norm used in Barron's original papers (see for example \cite{barron1993universal}).
\end{remark}

As a consequence of the {H\"older's inequality}, we trivially have
\begin{equation}\label{eqn: holder}
  \mathcal{B}_{\infty} \subset \cdots \mathcal{B}_{2} \subset \mathcal{B}_{1}. \nonumber
\end{equation}
However, the opposite is also true for the ReLU activation function we are considering.
\begin{proposition}
\label{pro: barron-space-eq}
For any $f\in\cB_1$, we have $f\in\cB_\infty$ and 
\[
\|f\|_{\cB_1} = \|f\|_{\cB_{\infty}}.
\]
\end{proposition}

As a consequence, we have that for any $1\leq p\leq \infty$,  $\cB_{p}=\cB_\infty$ and $\|f\|_{\cB_p}=\|f\|_{\cB_{\infty}}$. 
Hence, we can use $\cB$ and $\|\cdot\|_{\cB}$ to denote the Barron space and Barron norm.

A natural question is: 
What kind of functions  are in the Barron space?
The following is a restatement of an important result proved in \cite{klusowski2016risk}.
 It is an extension of the Fourier analysis of two-layer sigmoidal neural networks in Barron's seminal work \cite{barron1993universal}.

\begin{proposition}[Theorem 6 in~\cite{klusowski2016risk}]\label{thm:specnorm}
Let { $f\in C(X)$}, the space of continuous functions on $X$, and assume that $f$ satisfies:
$$
\gamma(f):=\inf_{\hat{f}}\int_{\RR^d} \|\omega\|_1^2 |\hat{f}(\omega) | d \omega < \infty,
$$
where $\hat{f}$ is the Fourier transform of an extension of $f$ to $\RR^d$. Then $f$ admits a representation
as in \eqref{integral1}.
Moreover,  
\begin{equation}\label{eq:specnorm}
\|f\|_{\mathcal{B}} \le 2\gamma(f) + 2\|\nabla f(0)\|_1 + 2 |f(0)|.
\end{equation}
\end{proposition}

\begin{remark}
In Section 9 of~\cite{barron1993universal}, examples of functions with bounded $\gamma(f)$ are given (e.g. Gaussian, positive definite functions,  etc.). \cite{barron1993universal} used the norm
$
\int_{\RR^d} \|\omega\| |\hat{f}(\omega) | d \omega,
$
 instead of  $\gamma(f)$, but the analysis also shows that Gaussian and positive definite functions give rise to finite values of $\gamma(f)$.
By Proposition~\ref{thm:specnorm}, these functions belong to the Barron space. 
\end{remark}

In addition, the Barron space is also closely related to a family of RKHS.  { Let $\bw = (\bb,c)$. Due to the scaling invariance of $\sigma(\cdot)$, we can assume $\bw\in \mathbb{S}^d$.} Then \eqref{integral1} can be written as 
\begin{equation}\label{eqn: express_pi}
   f(\bx) = \int_{\SS^{d}} a \sigma(\bw^T\tilde{\bx}) \rho(d a,d\bw) = \int_{\SS^d} a(\bw)\sigma(\bw^T\tilde{\bx}) \pi(d\bw),
\end{equation}
\[
   a(\bw) = \frac{\int_{\RR} a \rho(a,\bw)da}{{ \pi(\bw)}}, \quad
   { \pi(\bw) = \int_{\RR} \rho(a,\bw) da }
\]
Moreover,
\begin{equation}\label{eqn: def-b2norm}
  \|f\|_{\cB_2}^2 = \inf_{\pi} \EE_{\pi}[|a(\bw)|^2], \nonumber
\end{equation}
{where the infimum is taken over all $\pi$ that satisfies~\eqref{eqn: express_pi}.}

Given a fixed probability distribution $\pi$, we can define a kernel:
$$
k_{\pi} (\bx,\bx')=\mathbb{E}_{w\sim\pi}[\sigma(\bw^T\tilde{\bx})\sigma(\bw^T\tilde{\bx}')]
$$
Let $\cH_{k_{\pi}}$ denote the RKHS induced by $k_{\pi}$. Then we have the following proposition.

\begin{proposition}\label{prop: union_rkhs}
$$ 
\mathcal{B} = \bigcup_{\pi\in P(\SS^{d})} \cH_{k_{\pi}} .
$$
\end{proposition}

\subsection{Direct and inverse approximation theorems}

With \eqref{integral1}, approximating $f$ by two-layer networks becomes 
a Monte Carlo integration problem. 

\begin{theorem}\label{thm: barron-directapprox}
For any $f \in \mathcal{B}$ and $m>0$, there exists a two-layer neural network $f_m(\cdot; \Theta),
f_m(\bx; \Theta)= \frac{1}{m}\sum_{k=1}^m a_k\sigma(\bb_k^T\bx+c_k)$ ($\Theta$ denotes the parameters $\{(a_k, \bb_k, c_k), k \in [m] \}$ in the neural network),
such that
$$
\|f(\cdot)- f_m(\cdot;\Theta) \|^2 \le  \frac{3 \|f \|^2_{\mathcal{B}}}{m},
$$
Furthermore, we have
\[
\|\Theta\|_{\cP}:=\frac{1}{m}\sum_{j=1}^m|a_j|(\|\bb_j\|_1+|c_j|) \leq 2\|f\|_{\cB}.
\]
\end{theorem}
\begin{remark}
{ We call $\|\Theta\|_{\cP}$ the path norm of two-layer neural network. This is the analog of the Barron norm of functions in $\cB$. Hence, when studying approximation properties, it is natural to study two-layer neural networks with bounded path norm. }
\end{remark}

One can also prove an inverse approximation theorem.
To state this result, we define:
\[
\mathcal{N}_{Q} = \left\{\,\frac{1}{m}\sum_{k=1}^m a_k\sigma(\bb_k^T\bx+c_k) : 
\frac{1}{m}\sum_{k=1}^m |a_k|(\|\bb_k\|_1+|c_k|) \leq Q, m\in \mathbb{N}^{+}\,\right\}.
\]
\begin{theorem}\label{thm: barron-invapprox}
Let $f^*$ be a continuous function on $X$. 
Assume there exists a constant $Q$ and a sequence of functions  $(f_m) \subset \mathcal{N}_{Q} $
such that
$$ 
f_m(\bx) \rightarrow f^*(\bx) 
$$
for all $\bx \in X$. Then
there exists a probability distribution $\rho^*$ on {$(\Omega, \Sigma_{\Omega})$}, such that
\[
   f^*(\bx) = \int a\sigma(\bb^T\bx+c) \rho^*(da,d\bb,dc),
\]
for all $\bx \in X$.
Furthermore, we have
$f^* \in \mathcal{B}$ with
\[
\|f^* \|_{\mathcal{B}} \le Q.
\]
\end{theorem}

\subsection{Estimates of the Rademacher complexity}
Next, we show that the Barron spaces we defined  have low complexity. We show this by bounding the Rademacher complexity of bounded sets in the Barron spaces.
\begin{definition}[Rademacher complexity]
Given a set of functions $\cF$ and $n$ data samples $S=\{\bx_1, \bx_2, \cdots, \bx_n$\}, the Rademacher complexity of $\cF$ with respect to $S$ is defined as 
\begin{equation*}
    \rad_n(\cF)=\frac{1}{n}\EE_{\bm{\xi}}\sup_{f\in\cF}\sum\limits_{i=1}^n \xi_if(\bx_i),
\end{equation*}
where $\bm{\xi}=(\xi_1, \xi_2, \cdots, \xi_n)$ is a vector of $n$ \textit{i.i.d.} random variables that satisfy $\PP(\xi=1)=\PP(\xi=-1)=\frac{1}{2}$.
\end{definition}

The following theorem gives an estimate of the Rademacher complexity of the Barron space. 
Similar results can be found in \cite{bach2017breaking}.  We include the proof in the next section for completeness.

\begin{theorem}\label{thm:barron_rad}
Let $\cF_Q = \{ f \in \mathcal{B} : \|f\|_{\mathcal{B}} \le Q  \}$.
Then we have
\begin{align}
    \rad_n(\cF_Q) \leq 2Q \sqrt{\frac{2\ln(2d)}{n}} \nonumber
\end{align}
\end{theorem}

{From Theorem 8 in \cite{bartlett2002rademacher}, we see that the above result implies that functions in the Barron spaces can  be learned efficiently }.

\subsection{Barron space for Non-ReLU functions and the space $\mathcal{F}_1$}

{
The definition of the Barron space and Barron norm can be extended to representations~\eqref{integral1} with $\sigma(\cdot)$ being a general activation function. Specifically, for any function $f$ with representation
\begin{equation} 
f(\bx) = \int_{\Omega} a \tilde{\sigma} (\bb^T  \bx + c) \rho (da, d\bb, dc), \quad \bx \in X,
\label{integral_general}
\end{equation}
where $\tilde{\sigma}$ is an activation function not necessarily ReLU, we define the extended Barron norm (which is denoted by $\|\cdot\|_{\tilde{\cB}_p}$) as
\begin{align}
    \|f\|_{\tilde{\cB}_p}:
    &=\inf_{\rho}\ (\mathbb{E}_{\rho}\left[|a|^p(\|\bb\|_1+|c|+1)^p\right])^{1/p},\label{eqn: extended_barron}
\end{align}
where  $p\in [1, \infty]$, and the infimum is taken over all $\rho$ for which~\eqref{integral_general} holds. The extended Barron space $\tilde{\cB}_p$ is defined as the set of functions with finite $\tilde{\cB}_p$ norm. In this case, since the homogeneity property does not hold for the activation function, $\tilde{\cB}_p$ spaces with different $p$ are not equal. The direct approximation theorem and Rademacher complexity control can be proven for $\tilde{\cB}_p$ as long as $\tilde{\sigma}$ satisfies
\begin{equation*}
     \int_\bR |\tilde{\sigma}''(x)|(|x|+1)dx<\infty. 
\end{equation*}
See~\cite{li2020complexity} for more details. 

We deal with general activation functions by approximating them using two-layer ReLU neural networks, and the ``$+1$'' term in~\eqref{eqn: extended_barron} appears naturally during the approximation process. It is worth mentioning that if $\tilde{\sigma}=$ReLU the $\tilde{\cB}_p$ norms become equivalent with the Barron norm $\|\cdot\|_\cB$, because of the infimum and the homogeneity property.

In \cite{bach2017breaking},  a similar function space $\cF_1$ is defined by using  the variation norm~\cite{kurkova2001bounds,mhaskar2004tractability}. In \cite{bach2017breaking}, signed measures are used to represent the function as follows, 
\begin{equation}\label{eqn: bach-representation}
    f(x) = \int_{\mathcal{V}} \sigma(\bb^T\bx+c) d\mu(\bb,c),
\end{equation}
where $\mathcal{V}$ is the support of the signed measure $\mu$. Let $S_f$ denote the set of signed measures such that \eqref{eqn: bach-representation} holds. 
The $\cF_1$ norm of $f$ is given by 
\begin{equation*}
    \begin{aligned}
    \|f\|_{\cF_1}:= \inf_{\mu\in S_f}|\mu|(\mathcal{V}),
    \end{aligned}
\end{equation*}
where $|\mu|(\mathcal{V})$ denotes the total variation of $\mu$. The estimate of Rademacher complexity of $\cF_1$ is provided for the ReLU activation function.

For ReLU activation function, $\cF_1$ is equivalent with $\cB$, and the norms are equal, too~\cite{wojtowytsch2020representation}. However, for a general activation function (e.g. tanh, sigmoid), the Barron space is different from $\cF_1$.
$\cF_1$ typically requires $(\bb,c)$ to lie in a compact set, which is generally not true. With $(\bb,c)$ being in a compact set, the variation norm only considers $a$ and treat features with any $(\bb,c)$ equivalently. Hence, a very simple feature will have the same variation norm as a complicated feature, which leads to loose bounds for simple functions. On the contrary, the $\tilde{\cB}_p$ norms consider $(a,\bb,c)$ together, and features with different $(\bb,c)$ make different contributions to the norm. 
}

\subsection{Proofs}
\subsubsection{Proof of Proposition \ref{pro: barron-space-eq}}
Take $f\in\cB_1$. { For any $\varepsilon>0$,} there exists a probability measure $\rho$ that satisfies
\begin{equation} 
f(\bx) = \int_{\Omega} a \sigma (\bb^T  \bx + c) \rho (da, d\bb, dc), \quad \forall \, \bx \in X, \nonumber
\end{equation}
and 
\begin{equation}
\EE_\rho \left[ |a|(\|\bb\|_1+|c|) \right]< \|f\|_{\cB_1}+\varepsilon.\nonumber
\end{equation}
Let $\Lambda=\{(\bb,c):\ \|\bb\|_1+|c|=1\}$, and consider two measures $\rho_{+}$ and $\rho_{-}$ on $\Lambda$ defined by 
\begin{eqnarray}
\rho_{+}(A)=\int_{\{(a,\bb,c):\ (\hat{\bb},\hat{c})\in A, a>0\}} |a|(\|\bb\|_1+|c|)\rho(da,d\bb,dc),\nonumber\\
\rho_{-}(A)=\int_{\{(a,\bb,c):\ (\hat{\bb},\hat{c})\in A, a<0\}} |a|(\|\bb\|_1+|c|)\rho(da,d\bb,dc),\nonumber
\end{eqnarray}
for any Borel set $A\subset\Lambda$, where 
\begin{equation}
\hat{\bb}=\frac{\bb}{\|\bb\|_1+|c|},\ \ \hat{c}=\frac{c}{\|\bb\|_1+|c|}.\nonumber
\end{equation}
Obviously  $\rho_{+}(\Lambda)+\rho_{-}(\Lambda)=\EE_\rho \left[ |a|(\|\bb\|_1+|c|) \right]$, and
\begin{equation}
f(\bx) = \int_{\Lambda} \sigma(\bb^T\bx+c)\rho_{+}(d\bb,dc)-\int_{\Lambda} \sigma(\bb^T\bx+c)\rho_{-}(d\bb,dc).\nonumber
\end{equation}
Next, we define extensions of  $\rho_{+}$ and $\rho_{-}$ to $\{-1,1\}\times\Lambda$ by 
\begin{eqnarray}
&&\tilde{\rho}_{+}(A')=\rho_{+}(\{(\bb,c):\ (1,\bb,c)\in A'\}), \nonumber\\
&&\tilde{\rho}_{-}(A')=\rho_{-}(\{(\bb,c):\ (-1,\bb,c)\in A'\}),\nonumber
\end{eqnarray}
for any Borel sets $A'\subset \{-1,1\}\times\Lambda$, and let $\tilde{\rho}=\tilde{\rho}_{+}+\tilde{\rho}_{-}$. Then we have $\tilde{\rho}(\{-1,1\}\times\Lambda)=\EE_\rho \left[ |a|(\|\bb\|_1+|c|) \right]$ and
\begin{equation}
f(\bx) = \int_{\{-1,1\}\times\Lambda}a\sigma(\bb^T\bx+c)\tilde{\rho}(da,d\bb,dc).\nonumber
\end{equation}
Therefore, we can normalize $\tilde{\rho}$ to be a probability measure, and
\begin{equation}
    \|f\|_{\cB_\infty} \leq \tilde{\rho}(\{-1,1\}\times\Lambda){ \leq\|f\|_{\cB_1}+\varepsilon}.\nonumber
\end{equation}
Taking the limit as $\varepsilon\rightarrow0$, we have $\|f\|_{\cB_\infty}\leq\|f\|_{\cB_1}$. Since $\|f\|_{\cB_1}\leq \|f\|_{\cB_{\infty}}$ from H\"older's inequality, we conclude that $\|f\|_{\cB_1}=\|f\|_{\cB_{\infty}}$.
\qed

\subsubsection{Proof of Theorem~\ref{prop: union_rkhs}}
According to  \cite{rahimi2008uniform}, we have the following characterization of $\cH_{k_\pi}$:
\[
  \cH_{k_\pi} = \left\{\int_{\SS^d} a(\bw) \sigma(\bw^T\tilde{\bx}) d \pi(\bw) : \EE_{\pi}[|a(\bw)|^2]<\infty\right\}.
\]
In addition, for any $h\in \cH_{k_\pi}$, $\|h\|^2_{\cH_{k_{\pi}}} = \EE_{\pi}[|a(\bw)|^2]$. It is obvious that 
for any $\pi \in P(\SS^d)$, 
$
\cH_{k_{\pi}}  \subset \mathcal{B}_2
$, which implies that $\cup_{\pi} \cH_{k_{\pi}}\subset \cB_2$. 
Conversely, for any $f\in \cB_2$, { there exists a probability distribution $\tilde{\pi}$ that satisfies
\begin{equation}
  f(\bx)=\int_{\SS^d} a(\bw)\sigma(\bw^T\tilde{\bx})\tilde{\pi}(d\bw) \quad \forall \bx \in X, \nonumber
\end{equation} 
and $\EE_{\tilde{\pi}}[|a(\bw)|^2]\leq2\|f\|^2_{\cB_2}<\infty$. 
Hence  we have $f\in\cH_{k_{\tilde{\pi}}}$, which implies $\cB_2\subset \cup_{\pi} \cH_{k_\pi}$.
}Therefore $\cB_2=\cup_{\pi} \cH_{k_{\pi}}$. 
Together with Proposition~\ref{pro: barron-space-eq}, we  complete the proof.
\qed

\subsubsection{Proof of Theorem \ref{thm: barron-directapprox}}

Let $\varepsilon$ be a positive number such that $\varepsilon < 1/5$.
 Let $\rho$ be a probability distribution such that ${ f(\bx)=\EE_{\rho}[a\sigma(\bb^T\bx+c)]}$ and $\EE_{\rho}[|a|^2(\|\bb\|_1+|c|)^2]\leq(1+\varepsilon)\|f\|^2_{\cB_2}$.
Let $\phi(\bx;\theta)=a \sigma(\bb^T\bx+c)$ with $\theta=(a,\bb,c)\sim \rho$. 
Then we have $\EE_{\theta \sim \rho}[\phi(\bx;\theta)]=f(\bx)$.
Let $\Theta=\{\theta_j\}_{j=1}^m$ be i.i.d.\ random variables drawn from $\rho(\cdot)$, and consider the following empirical average,  
\[
    \hat{f}_m(\bx;\Theta) = \frac{1}{m}\sum_{j=1}^m\phi(\bx;\theta_j).
\]
Let $\cE(\Theta)=\EE_{\bx}[|\hat{f}_m(\bx;\Theta)-f(\bx)|^2]$ be the approximation error. Then we have
\begin{align*}
    \EE_{\Theta}[\cE(\Theta)] &= \EE_{\Theta}\EE_{\bx}|\hat{f}_m(\bx;\Theta)-f(\bx)|^2 \\
    &=  \EE_{\bx}\EE_{\Theta}|\frac{1}{m}\sum_{j=1}^m\phi(\bx;\theta_j)-f(\bx)|^2\\
    &= \frac{1}{m^2}\EE_{\bx}\sum_{j,k=1}^m \EE_{\theta_j,\theta_k}[(\phi(\bx;\theta_j) - f(\bx)) (\phi(\bx;\theta_k) - f(\bx)) ]\\
    &\leq \frac{1}{m^2} \sum_{j=1}^m\EE_{\bx}\EE_{\theta_j}[(\phi(\bx;\theta_j)-f(\bx))^2] \\
    &\leq \frac{1}{m}  \EE_{\bx}\EE_{\theta\sim \rho}[\phi^2(\bx;\theta)] \\
    &\leq \frac{{ (1+\varepsilon)}\|f\|^2_{\cB_2}}{m}.
\end{align*}
In addition, 
\[
\EE_{\Theta}[\|\Theta\|_{\cP}]=\frac{1}{m}\sum_{j=1}^m \EE_{\Theta}[\|a_j\|(\|\bb_j\|_1+|c_j|)] \leq { (1+\varepsilon)}\|f\|_{\cB_2}.
\]

Define the event $E_1=\{\cE(\Theta) < \frac{3 \|f\|^2_{\cB_2}}{m}\}$, and $E_2=\{\|\Theta\|_{\cP} < 2\|f\|_{\cB_2}\}$.
By Markov inequality, we have 
\begin{align*}
    \PP\{E_1\} &=1 -\PP\{E_1^c\} \geq  1 - \frac{\EE_{\Theta}[\cE(\Theta)]}{3\|f\|_{\cB_2}^2/m}\geq \frac{{ 2-\varepsilon}}{3} \\
    \PP\{E_2\}&= 1- \PP\{E_2^c\} \geq  1- \frac{\EE_{\Theta}[\|\Theta\|_{\cP}]}{2\|f\|_{\cB_2}}\geq \frac{{ 1-\varepsilon}}{2}.
\end{align*}
Therefore we have
\begin{align*}
\PP\{E_1\cap E_2\} = \PP\{E_1\}+\PP\{E_2\} -1 \geq \frac{{ 2-\varepsilon}}{3}+ \frac{{ 1-\varepsilon}}{2}-1 = \frac{1-5\varepsilon}{6}>0.
\end{align*} 
Choose any $\Theta$ in  $ E_1\cap E_2$.   The two-layer neural network model defined by this $\Theta$ satisfies both
requirements in the theorem.
\qed

\subsubsection{Proof of Theorem \ref{thm: barron-invapprox}}
Without loss of generality, we assume that $\|\bb\|_1+|c| = 1$, otherwise due to the scaling invariance of $\sigma(\cdot)$ we can redefine the parameters as follows,
\[
  a \leftarrow a (\|\bb\|_1 + |c|),\quad \bb \leftarrow \frac{\bb}{\|\bb\|_1+|c|}, \quad c \leftarrow \frac{c}{\|\bb\|_1+|c|}.
\]
Let $\Theta_m=\{(a_{k}^{(m)},\bb_{k}^{(m)},c_{k}^{(m)})\}_{k=1}^m$ be the parameters in the two-layer neural network model $f_m$ and let  $A=\sum_{k=1}^m |a_k|$ and $\alpha_k = \frac{|a_k|}{A}$.
Then we can define a probability measure:
\[
  \rho_m = \sum_{k=1}^m \alpha_k \delta\left(a-\frac{\sgn(a_{k}^{(m)}) A}{m} \right)\delta(\bb-\bb_{k}^{(m)})\delta(c-c_{k}^{(m)}),
\]
which satisfies 
\[
  f_m(\bx;\Theta_m) = \int a\sigma(\bb^T\bx+c) \rho_m(da,d\bb,dc).
\]

Let
\[
  K_Q =\{(a,\bb,c) : |a|\leq Q, \|\bb\|_1+|c|\leq 1\}.
\]
It is obvious that $\text{supp}(\rho_m)\subset K_Q$ for all $m$.
Since $K_{Q}$ is compact, the sequence of probability measure $(\rho_m)$ is tight. By Prokhorov's Theorem, there exists a subsequence $(\rho_{m_k})$ and a probability measure $\rho^*$ such that $\rho_{m_k}$ converges weakly to $\rho^*$.

The fact that $\text{supp}(\rho_m)\subset K_Q$ implies $\text{supp}(\rho^*)\subset K_Q$. Therefore, we  have 
\[
  \|f^*\|_{\cB} = \|f^*\|_{\cB_{\infty}}\leq Q.
\]
For any $\bx\in X$,    $a\sigma(\bb^T\bx+c)$ is continuous with respect to $(a,\bb,c)$ and bounded from above by $Q$. Since $\rho^*$ is the weak limit of $\rho_{m_k}$, we have 
\[
  f^*(\bx) = \lim_{k\to \infty} \int a\sigma(\bb^T\bx+c) d\rho_{m_k} = \int a\sigma(\bb^T\bx+c) d\rho^*(da,d\bb,dc).
\]
\qed 

\subsubsection{Proof of Theorem~\ref{thm:barron_rad}}
\label{sec: proof-barron-rad}
Let $\bw=(\bb^T,c)^T$ and $\tilde{\bx}=(\bx^T,1)^T$. { For any $\varepsilon>0$ and $f\in\cB$, let $\rho_f^\varepsilon(a,\bw)$ be a distribution such that $f(\bx)=\EE_{\rho_f^\varepsilon}[a\sigma(\bb^T\bx+c)]$ and $\EE_{\rho_f^\varepsilon}[|a|\|\bw\|_1]<(1+\varepsilon)\|f\|_{\cB}$.  Then,} 
\begin{align}\label{eqn: ppp}
\nonumber n \rad_n(\cF_Q) 
&= \EE_{\bm{\xi}}[\sup_{f\in \cF_Q}\sum_{i=1}^n \xi_i \EE_{\rho_f^\varepsilon}[a\sigma(\bw^T\bx_i)]]\\
\nonumber &=\EE_{\bm{\xi}}[\sup_{f\in \cF_Q}\EE_{\rho_f^\varepsilon}[\sum_{i=1}^n \xi_i a\sigma(\bw^T\bx_i)]]\\
\nonumber &=\EE_{\bm{\xi}}[\sup_{f\in \cF_Q}\EE_{\rho_f^\varepsilon}[|a|\|\bw\|_1|\sum_{i=1}^n \xi_i\sigma(\hat{\bw}^T\bx_i)|]]\\
&\leq (1+\varepsilon)Q \EE_{\bm{\xi}}[\sup_{\|\bw\|\leq 1}|\sum_{i=1}^n \xi_i\sigma(\bw^T\bx_i)|].
\end{align}
{
Due to the symmetry, we have 
\begin{align}\label{eqn: qqq}
\notag \EE_{\bm{\xi}}[\sup_{\|\bw\|\leq 1}|\sum_{i=1}^n \xi_i\sigma(\bw^T\bx_i)|] &\leq \EE_{\bm{\xi}}[\sup_{\|\bw\|\leq 1}\sum_{i=1}^n \xi_i\sigma(\bw^T\bx_i)] + \EE_{\bm{\xi}}[\sup_{\|\bw\|\leq 1}-\sum_{i=1}^n \xi_i\sigma(\bw^T\bx_i)] \\
\notag &=2\EE_{\bm{\xi}}[\sup_{\|\bw\|\leq 1} \sum_{i=1}^n \xi_i\sigma(\bw^T\bx_i)] \\
&\leq 2\EE_{\bm{\xi}}[\sup_{\|\bw\|\leq 1} \sum_{i=1}^n \xi_i\bw^T\bx_i],
\end{align}
where the last inequality follows from the contraction property of Rademacher complexity (see Lemma 26.9 in \cite{shalev2014understanding}) and 
the fact that $\sigma(\cdot)$ is Lipschitz continuous with Lipschitz constant $1$. Applying Lemma 26.11 in \cite{shalev2014understanding} and plugging \eqref{eqn: qqq} into \eqref{eqn: ppp}, we obtain 
\begin{equation}
\rad_n(\cF_Q) \leq 2(1+ \varepsilon )Q\sqrt{\frac{2\ln(2d)}{n}}. \nonumber
\end{equation}
Taking $\varepsilon\to 0$, we complete the proof.
}
\qed

\section{Flow-induced function spaces}
\label{sec: cfs}

In this section, we carry out a similar program for residual neural networks. Since the limit of these networks give rise to continuous in time flows, the natural function spaces and norms associated with the residual neural networks are also flow-based. For this reason we call them flow-induced spaces and flow-induced norms, respectively.
Similar to what was done in the last section, we establish a natural connection between  these function spaces and
residual neural networks, by proving direct and inverse approximation theorems. We also prove a complexity bound 
for the flow-induced space. 

We postpone all the proofs to the end of this section.

\subsection{The compositional law of large numbers}
We consider residual neural networks defined by
\begin{eqnarray}
 \bz_{0,L}(\bx) &=& \bV\bx, \nonumber \\
 \bz_{l+1,L}(\bx) &=& \bz_{l,L}(\bx)+\frac{1}{L}\bU_l\sigma \circ(\bW_l z_{l,L}(\bx)), \nonumber \\
{ f_L(\bx;\Theta)} &=& \balpha^T\bz_{L,L}(\bx), \label{eq:resnet}
\end{eqnarray}
where $\bx\in\bR^d$ is the input,  $\bV \in \RR^{D\times d}, \bW_l\in\bR^{m\times D}$, $\bU_l\in\bR^{D\times m}, \balpha\in\bR^D$ and we use $\Theta:=\{\bV, \bU_1,\dots,\bU_L, \bW_l,\dots,\bW_L, \balpha\}$ to denote all the parameters to be learned from data. Without loss of generality, we will fix $\bV$ to be
\begin{equation}\label{eq:bV}
    \bV=\left[\begin{array}{l}
         I_{d\times d} \\
         0_{(D-d)\times d}
    \end{array}\right].
\end{equation}
We will fix $D$ and $m$ throughout this paper, and when there is no danger for confusion
 we will omit $\Theta$ in the notation and use $f_L(\bx)$ to denote the residual network for simplicity.

For two layer neural networks, if the parameters $\{a_k,\bb_k,c_k\}$ are i.i.d sampled from a probability distribution $\rho$, then we have
\begin{equation}
\frac{1}{m}\sum\limits_{k=1}^m a_k\sigma(\bb_k^T\bx+c_k)\rightarrow \int a\sigma(\bb^T\bx+c)\rho(da,d\bb,dc),\ \nonumber
\end{equation}
when $m\rightarrow\infty$ as a consequence of the law of large numbers.
To get some intuition in the current situation, we will first study a similar setting for residual networks in which $\bU_l$ and $\bW_l$ are i.i.d sampled from a probability distribution $\rho$  on $\bR^{D\times m}\times\bR^{m\times D}$. 
To this end, we will study the behavior of $\bz_{L,L}(\cdot)$ as $L\rightarrow\infty$. 
The sequence of mappings we obtain is the repeated composition of  many \textit{i.i.d.} random  near-identity maps.

The following theorem can be viewed as a compositional version of the  law of large numbers. 
The ``compositional mean'' is defined with the help of the following ordinary differential equation (ODE) system:
 \begin{eqnarray}
\bz(\bx,0) &=& \bV\bx, \nonumber \\
\frac{d}{dt}\bz(\bx,t) &=& \mathbb{E}_{(\bU,\bW)\sim\rho} \bU \sigma(\bW\bz(\bx,t)). \label{eq:mean_ode}
\end{eqnarray}

\begin{theorem}\label{thm:lln}
Assume that  $\sigma$ is Lipschitz continuous and
\begin{equation}\label{eq:cond_lln}
\mathbb{E}_\rho \||\bU||\bW|\|_F^2<\infty.
\end{equation}
Then, the ODE~\eqref{eq:mean_ode} has a unique solution. 
For any $\bx\in X$,  we have
\begin{equation}
\bz_{L,L}(\bx)\rightarrow \bz(\bx,1) \nonumber
\end{equation}
in probability as $L\rightarrow+\infty$.
{ Moreover, we have 
\begin{equation}
    \lim_{L\to\infty}\sup_{\bx\in X} \EE\|\bz_{L,L}(\bx)-\bz(\bx,1)\|^2 = 0, \nonumber
\end{equation}
i.e.  the convergence is uniform with respect to $\bx\in X$.
}
\end{theorem}

This result can be extended to situations when the distribution $\rho$ is time-dependent,
which is the right setting in applications.

\begin{theorem}\label{thm:lln2}
Let $\{\rho_t,\ t\in[0,1]\}$ be a family of probability distributions  on $\bR^{D\times m}\times\bR^{m\times D}$ 
with the property that there exist constants $c_1$ and $c_2$ such that
\begin{equation}\label{eq:cond_lln2}
\mathbb{E}_{\rho_t} \||\bU||\bW|\|_F^2<c_1 \nonumber
\end{equation}
and 
\begin{equation}\label{eq:cond_lln_lip}
\left|\EE_{\rho_t} U\sigma(W\bz)-\EE_{\rho_s} U\sigma(W\bz)\right|\leq c_2|t-s||\bz| \nonumber
\end{equation}
 for all $s,t\in[0,1]$. 
Let $\bz$ be the solution of the following ODE,
\begin{eqnarray*}
\bz(\bx,0) &=& \bV\bx, \\
\frac{d}{dt}\bz(\bx,t) &=& \mathbb{E}_{(\bU,\bW)\sim\rho_t} \bU\sigma (\bW\bz(\bx,t)).
\end{eqnarray*}
{
Then, for any fixed $\bx \in \bX$, we have
\begin{equation}
\bz_{L,L}(\bx)\rightarrow \bz(\bx,1) \nonumber
\end{equation}
in probability as $L\rightarrow+\infty$. Moreover, the convergence is uniform in $\bx$.
}
\end{theorem}

{ Similar results have been proved in the context of stochastic approximations, for example in~\cite{kushner2003stochastic, benveniste2012adaptive}.}

\subsection{The flow-induced function spaces}
Motivated by the previous results, we consider the set of functions $f_{\balpha,\{\rho_t\}}$ defined by:
\begin{eqnarray}
\bz(\bx,0) &=& \bV\bx, \nonumber \\
\dot{\bz}(\bx,t) &=& \EE_{(\bU,\bW)\sim\rho_t} \bU\sigma(\bW\bz(\bx,t)), \nonumber \\
f_{\balpha,\{\rho_t\}}(\bx) &=& \balpha^T \bz(\bx,1), 
\label{eq:def_zx}
\end{eqnarray}
where $\bV\in\bR^{D\times d}$ { is given in~\eqref{eq:bV}}, $\bU\in\bR^{D\times m}$, $\bW\in\bR^{m\times D}$, and $\balpha\in\bR^{D}$.
To define a norm for these functions, we consider the following linear ODEs  ($p \ge 1$)
\begin{eqnarray}
  N_p(0) &=& \be, \nonumber \\
  \dot{N}_p(t) &=&  \left( \EE_{\rho_t}(|\bU||\bW|)^p \right)^{1/p}N_p(t), \label{eq:n_ode}
\end{eqnarray}
where $\be$ is the all-one vector in $\bR^D$. 
{Note that in~\eqref{eq:n_ode}, $|\bA|$ and $|\bA|^q$ are defined element-wise for matrix $\bA$, and the multiplication of $|\bU|$ and $|\bW|$ is the regular matrix multiplication.}
This linear system of equations has a unique solution as long as the expected value is integrable as a function of $t$.
If $f$ admits a representation as in \eqref{eq:def_zx}, we can define the $\cD_p$ norm of $f$. 
\begin{definition}\label{def:dp}
Let $f$ be a function that satisfies $f=f_{\balpha,\{\rho_t\}}$ for a pair of ($\balpha, \{\rho_t\}$), then we define
\begin{equation}
  \|f\|_{\cD_p(\balpha,\{\rho_t\})} = |\balpha|^T N_p(1),\nonumber
\end{equation}
to be the $\cD_p$ norm of $f$ with respect to the pair ($\balpha$, $\{\rho_t\}$). We define
\begin{equation}
  { \|f\|_{\cD_p} = \inf_{f=f_{\balpha,\{\rho_t\}}} |\balpha|^T N_p(1).} \label{eq:comp_norm}
\end{equation}
to be the $\cD_p$ norm of $f$. 
\end{definition}

As an example, if $\rho$ is constant in $t$, then the $\cD_p$ norm becomes
\begin{equation}
  { \|f\|_{\cD_p}=\inf_{f=f_{\balpha,\rho}} |\balpha|^T e^{(\mathbb{E}_\rho(|\bU||\bW|)^p)^{1/p}}\be }.\nonumber
\end{equation}
Given this definition, the flow-induced function spaces on $X$ 
are defined as the set of continuous functions that can be represented as
$f_{\balpha,\{\rho_t\}}$  in \eqref{eq:def_zx} with 
finite $\cD_p$ norm.
Here we assume that for any $t\in[0,1]$, 
$\rho_t$ is a probability distribution defined on $(\Omega,\Sigma_\Omega)$, 
$\Omega=\bR^{D\times m}\times\bR^{m\times D}$, 
$\Sigma_\Omega$ is the Borel $\sigma$-algebra on $\Omega$. We use $\cD_p$ to denote these function spaces. It is easy to see $\cD_p\subset\cD_q$ for $p\geq q$.

Note that in the definitions above, the only condition on $\{\rho_t\}$ is the existence and uniqueness of $\bz$ defined by~\eqref{eq:def_zx}. Hence, $\{\rho_t\}$ can be discontinuous as a function $t$. However, the compositional law of large numbers,
which is the underlying reason behind the approximation theorem that we will discuss next (Theorem~\ref{thm:lln2}), requires $\{\rho_t\}$ to satisfy some continuity condition. 
To that end, we define the following ``Lipschitz coefficient'' and ``Lipschitz norm'' for $\{\rho_t\}$
\begin{definition}\label{def:lip}
Given a family of probability distribution $\{\rho_t,\ t\in[0,1]\}$, the ``Lipschitz coefficient'' of $\{\rho_t\}$, which is denoted by $\lip$, is defined as the infimum of all the number $L$ that satisfies
\begin{equation}\label{eq:lip}
\left| \EE_{\rho_t}\bU\sigma(\bW\bz)-\EE_{\rho_s}\bU\sigma(\bW\bz) \right|\leq \lip|t-s||\bz|, \nonumber
\end{equation}
and 
\begin{equation}\label{eq:lip2}
\left| \left\|\EE_{\rho_t}|\bU||\bW|\right\|_{1,1}-\left\|\EE_{\rho_s}|\bU||\bW|\right\|_{1,1} \right|\leq \lip|t-s|, \nonumber
\end{equation}
for any $t,s\in[0,1]$, where $\|\cdot\|_{1,1}$ is the sum of the absolute values of all the entries in a matrix.
The ``Lipschitz norm'' of $\{\rho_t\}$ is defined as
\begin{equation}
\|\{\rho_t\}\|_\Lip=\left\|\EE_{\rho_0}|\bU||\bW|\right\|_{1,1}+\lip. \nonumber
\end{equation}
\end{definition}

{With the Lipschitz norm of $\{\rho_t\}$ defined above, we can introduce another class of function spaces $\tilde{\cD}_p$, which independently controls $N_p(1)$ and $\|\{\rho_t\}\|_\Lip$. 
\begin{definition}\label{def:dp-2}
Let $f$ be a function that satisfies $f=f_{\balpha,\{\rho_t\}}$ for a pair of ($\balpha, \{\rho_t\}$), then we define
\begin{equation}
  \|f\|_{\tilde{\cD}_p(\balpha,\{\rho_t\})} = |\balpha|^T N_p(1)+\|N_p(1)\|_1-D + \|\{\rho_t\}\|_\Lip,\nonumber
\end{equation}
to be the $\tilde{\cD}_p$ norm of $f$ with respect to the pair ($\balpha$, $\{\rho_t\}$). We define
\begin{equation}
  \|f\|_{\tilde{\cD}_p} = \inf_{f=f_{\balpha,\{\rho_t\}}} \|f\|_{\tilde{\cD}_p(\balpha,\{\rho_t\})}. \nonumber
\end{equation}
to be the $\tilde{\cD}_p$ norm of $f$. The space $\tilde{\cD}_p$ is 
defined as the set of {all the continuous functions that admit the representation 
$f_{\balpha,\{\rho_t\}}$ in \eqref{eq:def_zx} with finite $\tilde{\cD}_p$ norm.}
\end{definition}

\begin{remark}
We add a ``$-D$'' term in the definition of $\tilde{\cD}_p$ norm because 
$\|N_p(1)\|_1\geq D$ and we want the norm of the zero function to be $0$. 
As was stressed earlier, we use the terminology ``norm'' loosely, and we do not 
care whether these are really norms. 
Strictly speaking, they are just some 
quantities that can be used to bound approximation/estimation errors. 
\end{remark}
}

Next, for residual networks~\eqref{eq:resnet}, we define  a parameter-based norm as a discrete analog of~\eqref{eq:comp_norm}. This is similar to the $l_1$ path norm of the residual network, which is studied in \cite{neyshabur2017, ma2019priori}
\begin{definition}\label{eq:l1_path}
For a residual network defined by~\eqref{eq:resnet} with parameters $\Theta=\{\balpha, \bU_l, \bW_l, l=0,1,\cdots,L-1\}$, we define the $l_1$ path norm of $\Theta$ to be
\begin{equation}
\|\Theta\|_{P}=|\balpha|^T\prod\limits_{l=1}^L\left(I+\frac{1}{L}|\bU_l||\bW_l|\right)\be. \nonumber
\end{equation}
\end{definition}
\noindent We can also define the analog of the $p$-norms for $p>1$ for residual networks. But in this paper we will only use the $l_1$ norm defined above. 

It is easy to see that $\tilde{\cD}_p\subset\cD_p$, and for any $f\in\tilde{\cD}_p$ we have $\|f\|_{\cD_p}\leq\|f\|_{\tilde{\cD}_p}$. Moreover, for any $1\leq q\leq p$, if $f\in\tilde{\cD}_p$, then we have $f\in\tilde{\cD}_q$ and $\|f\|_{\tilde{\cD}_q}\leq\|f\|_{\tilde{\cD}_p}$.
The next proposition states that Barron space is embedded in $\tilde{\cD}_1$.
\begin{proposition}\label{thm:barron}
     Assume that $D\geq d+2$ and $m\geq1$.
   For any function $f\in\cB$,    we have $f\in\tilde{\cD}_1$, and
    \begin{equation}
        \|f\|_{\tilde{\cD}_1}\leq 2\|f\|_\cB+1. \nonumber
    \end{equation}
    Moreover, for any $\varepsilon>0$, there exists $(\balpha,\{\rho_t\})$ such that $\rho_t$ is fixed for all $t$, $f=f_{\balpha,\{\rho_t\}}$, and
    \begin{equation}
      \|f\|_{\tilde{\cD}_1(\balpha,\{\rho_t\})}\leq 2\|f\|_{\mathcal{B}} +1+\varepsilon.  \nonumber
    \end{equation}
\end{proposition}

Similar to the results of Proposition~\ref{thm:barron}, we can prove that the composition of two Barron functions belongs to the flow-induced function space, and the norm is bounded by a polynomial of the norms of the two Barron functions. 
\begin{proposition}\label{lm:barron_comp}
Assume that $D\geq d+3$ and $m\geq 1$.
Assume that  $g: [0, 1]^d \rightarrow [0, 1] \in\cB$, $h: [0, 1] \rightarrow \bR^1 \in\cB_1$. Let $f=h\circ g$ be the composition of $g$ and $h$.  Then we have $f\in\cD_1$ and 
\begin{equation}
\|f\|_{\cD_1}\leq (\|h\|_{\cB}+1)(\|g\|_{\cB}+1). \nonumber
\end{equation}
\end{proposition}

In~\cite{eldan2016power}, the authors constructed a sequence of functions 
$\{f_d:\ \bR^d\rightarrow\bR\}$ whose spectral norms~\eqref{eq:specnorm} grow
exponentially with respect to $d$. They also showed that these functions
can be expressed as the composition of two functions (one from $\bR^d$ to $\bR$ and the other from $\bR$ to $\bR$) whose spectral norms depend only polynomially on the 
dimension $d$. By Proposition~\ref{lm:barron_comp}, the $\cD_1$ norms of $f_d$ are bounded by a polynomial of $d$. This shows that in the high dimensions, 
the flow-induced norm can be significantly smaller  than the spectral norm. 
Combined with the direct approximation theorem below, this implies that residual networks can better approximate some functions than two-layer networks.

\subsection{Direct and inverse approximation theorems}
We first prove the direct approximation theorem, which states that functions in $\tilde{\cD}_2$ can be approximated by a sequence of residual networks with a $1/L^{1-\delta}$ error rate for any $\delta\in(0,1)$, and the networks have uniformly bounded path norm.

\begin{theorem}\label{thm:direct_comp}
Let $f\in\tilde{\cD}_2$, $\delta \in (0,1)$. Then, there exists an absolute constant $C$, such that
for any
\begin{equation*}
L\geq C\left(m^4D^6\|f\|_{\tilde{\cD}_2}^5(\|f\|_{\tilde{\cD}_2}+D)^2\right)^{\frac{3}{\delta}},
\end{equation*}
there is an $L$-layer residual network $f_L(\cdot;\Theta)$ that satisfies
\begin{equation}
  \|f-f_L(\cdot;\Theta)\|^2\leq \frac{\|f\|_{\tilde{\cD}_2}^2}{L^{1-\delta}}, \nonumber
\end{equation}
and 
\begin{equation}
  \|\Theta\|_P\leq 9\|f\|_{\tilde{\cD}_1}. \nonumber
\end{equation}
\end{theorem}

We can also prove an inverse approximation theorem,  which states that any function that can be approximated by a sequence of well-behaved residual networks has to belong to the flow-induced space. 

\begin{theorem}\label{thm:inverse}
Let $f$ be a function defined on $X$. Assume that there is a sequence of residual networks
$\{f_L(\cdot;\Theta_L)\}_{L=1}^\infty$ such that $f_L(\bx;\Theta)\rightarrow f(\bx)$ for every $\bx\in X$ as $L \rightarrow \infty$.
Assume further that the parameters in $\{f_L(\cdot;\Theta)\}_{L=1}^\infty$ are (entry-wise) bounded by $c_0$.
Then, we have $f\in\cD_\infty$, and 
$$\|f\|_{\cD_\infty}\leq\frac{2e^{m(c_0^2+1)}D^2c_0}{m}$$
Moreover, if for some constant $c_1$, $\|f_L\|_{\cD_1}\leq c_1$ holds for 
all $L>0$, then we have 
$$\|f\|_{\cD_1}\leq c_1$$
\end{theorem}

\subsection{Bounds for the Rademacher complexity}
Our final result is an upper bound for the Rademacher complexity {involving the flow-induced norm}. Due to technical difficulties, in this part we consider a family of modified flow-induced function norms $\|\cdot\|_{\hat{\cD}_p}$, which is defined as
\begin{equation}
 \|f\|_{\hat{\cD}_p} = \inf_{f=f_{\balpha,\{\rho_t\}}} |\balpha|^T \hat{N}_p(1)+\|\hat{N}_p(1)\|_1-D +\|\{\rho_t\}\|_\Lip,  \nonumber
\end{equation}
where $\hat{N}_p(t)$ is given by
\begin{align*}
  \hat{N}_p(0) &= 2\be, \\
  \dot{\hat{N}}_p(t) &= 2\left( \EE_{\rho_t}(|\bU||\bW|)^p \right)^{1/p}\hat{N}_p(t).  
\end{align*}
Denote by $\hat{\cD}_p$ the space of functions with finite $\hat{\cD}_p$ norm. Then,  we have

\begin{theorem}\label{thm:rad_res}
Let $\hat{\cD}_p^Q=\{f\in\hat{\cD}_p: \|f\|_{\hat{\cD}_p}\leq Q\}$, then we have
\begin{equation}
\rad_n(\hat{\cD}_{2}^Q)\leq 18Q\sqrt{\frac{2\log(2d)}{n}}. \nonumber
\end{equation}
\end{theorem}

The difference between the definitions of the spaces $\hat{\cD}_p$ and $\cD_p$ lies in the factor 2 that appears in $\hat{N}_p$.
At this stage, we are not able to remove this factor.  It should be noted that this factor of 2 is also present in the ``weighted path norm''
introduced in~\cite{ma2019priori}.
If $\bU$, $\bW$ and $\hat{N}_p(t)$ are scalars, then $\hat{N}_p(t)$ can be upper bounded by $(N_p(t))^2$. However, in the vectorial case this bound does not always hold. Hence, it is unclear how the two spaces $\hat{\cD}_p$ and $\cD_p$ are related.  
Clearly we can also develop an approximation theory for the space $\hat{\cD}_p$, but we feel it is worthwhile to show that the space $\tilde{\cD}_p$ is sufficient for that purpose.
{We also point out here at this stage, we allow to use
different quantities (norms) to control the approximation and estimation errors.}



\subsection{Proofs}\label{ssec:comp_proof}

\subsubsection{Proof of Theorem~\ref{thm:lln}}

To prove convergence, let $t_{l,L}=l/L$, and consider $\be_{l,L}(\bx)=\sqrt{L}(\bz_{l,L}(\bx)-\bz(\bx,t_{l,L}))$. 
We will focus on fixed $\bx$ and from now on we omit the dependence on  $\bx$ in the notations
and write instead $\be_{l,L}$, $\bz_{l,L}$ and $\bz(t)$, for example. From the definition of $\bz(t)$, we have
\begin{eqnarray}
\bz(t_{l+1,L}) &=& \bz(t_{l,L})+\int_{t_{l,L}}^{t_{l+1,L}}\EE\bU\sigma(\bW\bz(t))dt \nonumber \\
  &=& \bz(t_{l,L})+\frac{1}{L}\bU_l\sigma(\bW_l\bz(t_{l,L}))+\frac{1}{L}\left(\bU_l\sigma(\bW_l\bz(t_{l,L}))-\EE \bU\sigma(\bW\bz(t_{l,L}))\right) \nonumber \\
  && + \left(\frac{1}{L}\EE \bU\sigma(\bW\bz(t_{l,L}))-\int_{t_{l,L}}^{t_{l+1,L}}\EE\bU\sigma(\bW\bz(t))dt\right). \label{eq:lln_pf1}
\end{eqnarray}
Since 
\begin{equation}\label{eq:lln_pf2}
    \bz_{l+1,L}=\bz_{l,L}+\frac{1}{L}\bU_l\sigma(\bW_l\bz_{l,L}),
\end{equation}
subtract~\eqref{eq:lln_pf1} from~\eqref{eq:lln_pf2} gives
\begin{eqnarray}
\be_{l+1,L} &=& \be_{l,L}+\frac{1}{\sqrt{L}}\left(\bU_l\sigma(\bW_l\bz_{l,L})-\bU_l\sigma(\bW_l\bz(t_{l,L}))\right) \nonumber \\
  && + \frac{1}{\sqrt{L}}\left(\bU_l\sigma(\bW_l\bz(t_{l,L}))-\EE \bU\sigma(\bW\bz(t_{l,L}))\right) \nonumber \\
  && + \frac{1}{\sqrt{L}}\left(\EE \bU\sigma(\bW\bz(t_{l,L}))-L\int_{t_{l,L}}^{t_{l+1,L}}\EE\bU\sigma(\bW\bz(t))dt\right). \nonumber
\end{eqnarray}
Define
\begin{eqnarray}
I_{l,L}&=&\frac{1}{\sqrt{L}}\left(\bU_l\sigma(\bW_l\bz_{l,L})-\bU_l\sigma(\bW_l\bz(t_{l,L}))\right), \nonumber\\
J_{l,L}&=&\frac{1}{\sqrt{L}}\left(\bU_l\sigma(\bW_l\bz(t_{l,L}))-\EE \bU\sigma(\bW\bz(t_{l,L}))\right), \nonumber\\
K_{l,L}&=&\frac{1}{\sqrt{L}}\left(\EE \bU\sigma(\bW\bz(t_{l,L}))-L\int_{t_{l,L}}^{t_{l+1,L}}\EE\bU\sigma(\bW\bz(t))dt\right). \nonumber
\end{eqnarray}
Then, we have
\begin{equation}\label{eq:lln_pf_e}
    \be_{l+1,L}=\be_{l,L}+I_{l,L}+J_{l,L}+K_{l,L}.
\end{equation}

Next, we consider $\|\be_{l,L}\|^2$. From~\eqref{eq:lln_pf_e}, we get
\begin{eqnarray}
\|\be_{l+1,L}\|^2 &=& \|\be_{l,L}\|^2 + \|I_{l,L}\|^2 + \|J_{l,L}\|^2 + \|K_{l,L}\|^2 \nonumber \\
  && + 2\be_{l,L}^TI_{l,L} + 2\be_{l,L}^TJ_{l,L} + 2\be_{l,L}^TK_{l,L} \nonumber \\
  && + 2I_{l,L}^TJ_{l,L} + 2I_{l,L}^TK_{l,L} + 2J_{l,L}^TK_{l,L} \nonumber \\
  &\leq& \|\be_{l,L}\|^2 + 3\|I_{l,L}\|^2 + 3\|J_{l,L}\|^2 + 3\|K_{l,L}\|^2 \nonumber \\
  && + 2\be_{l,L}^TI_{l,L} + 2\be_{l,L}^TJ_{l,L} + 2\be_{l,L}^TK_{l,L} \label{eq:lln_pf3}. 
\end{eqnarray}
We are going to estimate the expectation of the right hand side of~\eqref{eq:lln_pf3} term by term. 
First, note that $\EE \| |\bU||\bW|\|^2_F<\infty$, which means $\bz(t)$ is bounded for $t\in [0,1]$. Hence, we can find a constant $C>0$ that satisfies
\begin{equation}
    \EE \| |\bU||\bW|\|_F \leq C, \ \  \EE \| |\bU||\bW|\|^2_F \leq C, \nonumber
\end{equation}
and $\|\bz(t)\|\leq C$. 
In addition,  note that for any $l$, $\bU_l$ and $\bW_l$ are independent with $\be_{l,L}$. 
Therefore, for $\|I_{l,L}\|^2$, we have
\begin{eqnarray}
\EE \|I_{l,L}\|^2 &=& \frac{1}{L}\EE\left\|\bU\sigma(\bW\bz_{l,L})-\bU\sigma(\bW\bz(t_{l,L}))\right\|^2 \nonumber \\
  &\leq& \frac{1}{L^2}\EE \| |\bU_l||\bW_l||\be_{l,L}|\|^2 \nonumber \\
  &\leq& \frac{C}{L^2} \EE \|\be_{l,L}\|^2. \nonumber
\end{eqnarray}
For  the term $\|J_{l,L}\|^2$, we have
\begin{equation}
\EE\|J_{l,L}\|^2\leq \frac{1}{L}\EE\||\bU_l||\bW_l||\bz(t_{l,L})|\|^2\leq\frac{C^2}{L}. \nonumber
\end{equation}
For the term $\|K_{l,L}\|$, since $\EE \| |\bU||\bW|\|_F \leq C$ and $\|\bz\|\leq C$, we know that the Lipschitz constant of $\bz(t)$ is bounded by $C^2$. Hence,  we have
\begin{eqnarray}
 \|K_{l,L}\| &\leq& \sqrt{L}\left\| \int_{t_{l,L}}^{t_{l+1,L}}\EE (\bU\sigma(\bW\bz(t_{l,L}))-\bU\sigma(\bW\bz(t)))dt \right\| \nonumber \\
   &\leq&\frac{C^2}{\sqrt{L}}\int_{t_{l,L}}^{t_{l+1,L}} (t-t_{l,L}) \EE \||\bU||\bW| \| dt \nonumber \\
   &\leq& \frac{C^3}{L\sqrt{L}}, \nonumber
\end{eqnarray}
which implies that 
\begin{equation}
    \EE\|K_{l,L}\|^2\leq\frac{C^6}{L^3}. \nonumber
\end{equation}
Next, we consider $\be_{l,L}^TI_{l,L}$. We easily have
\begin{equation}
\EE \be_{l,L}^TI_{l,L}\leq \frac{1}{L}\EE \||\bU||\bW|\|_F\|\be_{l,L}\|^2   \leq\frac{C}{L}\EE \|\be_{l,L}\|^2. \nonumber
\end{equation}
For $\be_{l,L}^TJ_{l,L}$, by the independence of $\bU_l$, $\bW_l$ and $\be_{l,L}$, we have
\begin{equation}
    \EE \be_{l,L}^TJ_{l,L} = 0. \nonumber
\end{equation}
Finally, for $\be_{l,L}^TK_{l,L}$, we have
\begin{equation}
    \EE \be_{l,L}^TK_{l,L}\leq \frac{C^3}{L\sqrt{L}}\sqrt{\EE\|\be_{l,L}\|^2}\leq\frac{C^3}{L\sqrt{L}}(\EE\|\be_{l,L}\|^2+1). \nonumber
\end{equation}

Plugging all the estimates above into~\eqref{eq:lln_pf3}, we obtain
\begin{equation}
\EE\|\be_{l+1,L}\|^2\leq \left(1+\frac{2C}{L}+\frac{3C}{L^2}+\frac{2C^3}{L\sqrt{L}}\right)\EE\|\be_{l,L}\|^2+\frac{3C^2}{L}+\frac{3C^6}{L^3}+\frac{2C^3}{L\sqrt{L}}. \nonumber
\end{equation}
Hence there is an $L_0$ depending only on $C$, such that if  $L > L_0$, we have 
\begin{equation}
\EE\|\be_{l+1,L}\|^2\leq \left(1+\frac{3C}{L}\right)\EE\|\be_{l,L}\|^2+\frac{4C^2}{L}. \nonumber
\end{equation}
Since $\be_{0,L}=0$, by induction we obtain
\begin{equation}
    \EE \|\be_{L,L}\|^2\leq4C^2e^{3C}, \nonumber
\end{equation}
which means 
\begin{equation}
    \EE \|\bz_{L,L}-\bz(1)\|^2\leq \frac{4C^2e^{3C}}{L} \rightarrow 0, \nonumber
\end{equation}
when $L\rightarrow \infty$.
This implies that $\bz_{L,L}\rightarrow\bz(1)$ { in probability}. 
\qed

\subsubsection{Proof of Theorem~\ref{thm:lln2}}

The only modification required for the proof of Theorem~\ref{thm:lln2}  is in the estimate of $K_{l,L}$. 
Now $K_{l,L}$ becomes
\begin{equation}
K_{l,L}=\frac{1}{\sqrt{L}}\left(\EE_{\rho_{t_{l,L}}} \bU\sigma(\bW\bz(t_{l,L}))-L\int_{t_{l,L}}^{t_{l+1,L}}\EE_{\rho_t}\bU\sigma(\bW\bz(t))dt\right). \nonumber
\end{equation}
The conditions of the theorem still guarantee that $\bz(t)$ is  Lipschtiz continuous. Hence, we can find a constant $C'$ 
such that  $\bz(t)$ is $C'$-Lipschitz  and 
\begin{equation}
    \EE_{\rho_t} \||\bU||\bW|\|\leq C', \nonumber
\end{equation}
for any $t\in[0,1]$. Hence, 
\begin{eqnarray}
\|K_{l,L}\| &\leq& \sqrt{L}\int_{t_{l,L}}^{t_{l+1,L}} \left\| \EE_{\rho_{t_{l,L}}} \bU\sigma(\bW\bz(t_{l,L}))-\EE_{\rho_t}\bU\sigma(\bW\bz(t)) \right\|dt \nonumber \\
  &\leq& \sqrt{L}\int_{t_{l,L}}^{t_{l+1,L}} \left\| \EE_{\rho_{t_{l,L}}} \bU\sigma(\bW\bz(t_{l,L}))-\EE_{\rho_t}\bU\sigma(\bW\bz(t_{l,L})) \right\|dt \nonumber \\
  && +\sqrt{L}\int_{t_{l,L}}^{t_{l+1,L}} \left\| \EE_{\rho_{t}} \bU\sigma(\bW\bz(t_{l,L}))-\EE_{\rho_t}\bU\sigma(\bW\bz(t)) \right\|dt \nonumber \\
  &\leq& \frac{c_2C'}{L\sqrt{L}}+\frac{C'^2}{L\sqrt{L}}. \label{eq:lln2_pf1}
\end{eqnarray}
From~\eqref{eq:lln2_pf1} we know that in this case  $K_{l,L}$ is of the same order as that in Theorem~\ref{thm:lln}. 
We can then complete the proof following the same arguments as in the
proof of Theorem~\ref{thm:lln}.
\qed

\subsubsection{Proof of Proposition~\ref{thm:barron}}

Since $f\in\cB$, for any $\varepsilon>0$, there exists a distribution $\rho^{\varepsilon}$ that satisfies
\begin{align}
&f(\bx) = \int_{\Omega} a \sigma (\bb^T  \bx + c) \rho_{\varepsilon} (da, d\bb, dc) \nonumber\\
&\EE_{\rho_{\varepsilon}}[|a|(\|\bb\|_1+|c|)] \leq \|f\|_{\cB}+\varepsilon. \nonumber
\end{align}
Define $\hat{f}$ by
\begin{eqnarray}
z(\bx,0) &=& \left[\begin{array}{c}
     \bx \\ 1 \\ 0 
     \end{array}\right] \nonumber\\
\frac{d}{dt} z(\bx,t) &=& \mathbb{E}_{(a,\bb, c)\sim\rho_{\varepsilon}}\left[\begin{array}{c}
     0 \\ 0 \\ a 
\end{array}\right]\sigma([\bb^T,c, 0]z(\bx,t))\label{eqn: Barron_function2} \\
\hat{f}(\bx) &=& \be_{d+2}^Tz(\bx,1) \nonumber
\end{eqnarray}
Then, we can easily verify that $\hat{f}=f$. 
Using $\rho_{\varepsilon}$, we can define probability distribution  $\tilde{\rho}_{\varepsilon}$ on $\bR^{D\times m}\times\bR^{m\times D}$: $\tilde{\rho}_{\varepsilon}$ is concentrated 
on matrices of the form that appears in \eqref{eqn: Barron_function2}.
Consider $\|f\|_{\tilde{\cD}_1(\be_{d+2}, \{\tilde{\rho}_{\varepsilon}\})}$, we have 
\begin{align*}
\|f\|_{\tilde{\cD}_1(\be_{d+2},\{\tilde{\rho}_{\varepsilon}\})} &= \be_{d+2}^T \exp \left(\EE_{\rho}
\left[\begin{array}{ccc}
0 & 0 & 0 \\
0 & 0 & 0 \\
|a\bb^T| & |ac| & 0 \\
\end{array}\right]
\right)
\be + 
\left\|
 \exp \left(\EE_{\rho}
\left[\begin{array}{ccc}
0 & 0 & 0 \\
0 & 0 & 0 \\
|a\bb^T| & |ac| & 0 \\
\end{array}\right]\right)
\right\|_1-D\nonumber \\
& =  \be_{d+2}^T\left[\begin{array}{ccc}
I & 0 & 0 \\
0 & 1 & 0 \\
\EE_{\rho}|a\bb^T| & \EE_{\rho}|ac| & 1 \\
\end{array}\right]\be +
\left\|
\left[\begin{array}{ccc}
I & 0 & 0 \\
0 & 1 & 0 \\
\EE_{\rho}|a\bb^T| & \EE_{\rho}|ac| & 1 \\
\end{array}\right]
\right\|_1-D
\nonumber \\
 & =  2\EE_{\rho_{\varepsilon}}|a|(\|\bb\|_1+|c|)+1 \\
 &\leq 2\|f\|_{\cB}+2\varepsilon + 1.
\end{align*}
Therefore, we have
\begin{equation}
\|f\|_{\tilde{\cD}_1}\leq\|f\|_{\tilde{\cD}_1(\balpha,\tilde{\rho}_{\varepsilon})}\leq 2\|f\|_{\cB}+2\varepsilon+1. \nonumber
\end{equation}
Taking $\varepsilon\rightarrow0$, we get
\begin{equation}
    \|f\|_{\tilde{\cD}_1}\leq2\|f\|_{\cB}+1. \nonumber
\end{equation}

Besides, since $\{\tilde{\rho}_\varepsilon\}$ gives the same probability distribution for all $t\in[0,1]$, we have $\textit{Lip}_{\{\tilde{\rho}_\varepsilon\}}=0$.
\qed
\ \\

\subsubsection{Proof of Theorem~\ref{thm:direct_comp}}

For any $\varepsilon>0$, let 
\begin{equation}
  \sigma^\varepsilon(x)=\int_{\bR}\frac{1}{\sqrt{2\pi\varepsilon^2}}e^{-\frac{(x-y)^2}{2\varepsilon^2}}\sigma(y)dy. \nonumber
\end{equation}
Then we have
\begin{equation}
|\sigma^\varepsilon(x)-\sigma(x)|<\varepsilon,\ \ |(\sigma^\varepsilon(x))'|\leq1, \ \ |(\sigma^\varepsilon(x))''|\leq\frac{1}{\varepsilon}, \nonumber
\end{equation}
for all $x\in\bR$.
For a function $f\in\tilde{\cD}_2$, we are going to show that for sufficiently large $L$ there exists an $L$-layer residual network $f_L$ such that
\begin{equation}
\|f-f_L\|^2\leq\frac{\|f\|_{\tilde{\cD}_2}^2}{L^{1-\delta}}, \nonumber
\end{equation}.

{ To do this, assume that $\balpha$ and $\{\rho_t\}$ satisfy $f=f_{\balpha,\{\rho_t\}}$ and $\|f\|_{\tilde{\cD}_2(\alpha,\{\rho_t\})}\leq2\|f\|_{\tilde{\cD}_1}$. Let $f_L$ be a residual network in the form~\eqref{eq:resnet}, and the weights $\bU_l$, $\bW_l$ are sampled from $\rho_{l/L}$. 
Let $f^\varepsilon$ and $f^\varepsilon_L$ be  generated in the same way as $f$ and $f_L$ using instead the activation function $\sigma^\varepsilon$. Then we have}
\begin{equation}\label{eq:direct_relu_pf1}
\|f-f_L\|^2\leq3\left(\|f-f^\varepsilon\|^2+\|f^\varepsilon-f^\varepsilon_L\|^2+\|f^\varepsilon_L-f_L\|^2\right).
\end{equation}

{Before dealing with~\eqref{eq:direct_relu_pf1}, we first prove the following lemma, which shows that we can pick the  family 
of distributions $\tilde{\rho}_t$ to have compact support.
\begin{lemma}\label{lm:equal}
For any $f\in\tilde{\cD}_1$ that satisfies the conditions of Theorem~\ref{thm:direct_comp}, and any $\varepsilon>0$, there exists $\balpha$ and $\{\rho_t\}$, such that  $f=f_{\balpha,\{\rho_t\}}$ and $\|f\|_{\tilde{\cD}_2(\balpha,\{\rho_t\})}\leq(1+\varepsilon)\|f\|_{\tilde{\cD}_2}$. Moreover, for any $t\in[0,1]$, we have
\begin{equation}
\max_{(\bU,\bW)\sim\rho_t}\left(\||\bU||\bW|\|_1\right)\leq (1+\varepsilon)\|f\|_{\tilde{\cD}_2}. \nonumber
\end{equation}
\end{lemma}
}

\noindent{\bf{Proof of Lemma~\ref{lm:equal}}}

The proof of Lemma~\ref{lm:equal} is similar to the proof of Proposition~\ref{pro: barron-space-eq}. By the definition of $\tilde{\cD}_2$, for any $f\in\tilde{\cD}_2$ and $\varepsilon>0$, there exists $\balpha$ and $\{\rho_t\}$ such that $f=f_{\balpha,\{\rho_t\}}$, $\|f\|_{\tilde{\cD}_2(\balpha,\{\rho_t\})}\leq(1+\varepsilon)\|f\|_{\tilde{\cD}_2}$, and hence $\|\{\rho_t\}\|_\Lip\leq (1+\varepsilon)\|f\|_{\tilde{\cD}_2}$. This means that for any $t\in[0,1]$, we have
\begin{equation}
    \left\|\EE_{\rho_t} |\bU||\bW|\right\|_1 \leq (1+\varepsilon)\|f\|_{\tilde{\cD}_2}. \nonumber
\end{equation}
Let $\Lambda=\{(\bU,\bW):\ \|\bW\|_1=1,\ \||\bU||\bW|\|_1=1\}$, and consider a family of measures $\{\rho^\Lambda_t\}$ defined by
\begin{equation}
  \rho^\Lambda_t(A)=\int_{(\bU,\bW):\ (\bar{\bU},\bar{\bW})\in\Lambda}\||\bU||\bW|\|_1\rho_t(d\bU,d\bW), \nonumber
\end{equation}
for any Borel set $A\subset\Lambda$, where 
\begin{equation}
  \bar{\bU}=\frac{\|\bW\|_1}{\||\bU||\bW|\|_1}\bU,\ \ \bar{\bW}=\frac{\bW}{\|\bW\|_1}. \nonumber
\end{equation}
It is easy to verify that $\rho^\Lambda_t(\Lambda)=\EE_{\rho_t}\||\bU||\bW|\|_1$ and
\begin{equation}
  \EE_{\rho^\Lambda_t}\bar{\bU}\sigma(\bar{\bW}\bz)=\EE_{\rho_t}\bU\sigma(\bW\bz) \nonumber
\end{equation}
hold for any $t\in[0,1]$ and $\bz\in\bR^D$.
After normalizing $\{\rho^\Lambda_t\}$, we obtain a family of  probability distributions $\{\tilde{\rho}^\Lambda_t\}$  on
\begin{equation}
\{(\bU,\bW):\ \|\bW\|_1=1,\ \||\bU||\bW|\|=\EE_{\rho_t}\||\bU||\bW|\|_1\}. \nonumber
\end{equation}
Finally, it is easy to verify that $f=f_{\balpha,\{\tilde{\rho}^\Lambda_t\}}$,  $\|f\|_{\tilde{\cD}_2(\balpha,\{\rho_t\})}\leq(1+\varepsilon)\|f\|_{\tilde{\cD}_2}$, as well as
\begin{equation}
\max_{(\bU,\bW)\sim\tilde{\rho}^\Lambda_t}\left(\||\bU||\bW|\|_1\right)\leq (1+\varepsilon)\|f\|_{\tilde{\cD}_2}. \nonumber
\end{equation}
\qed
\ \\

{From Lemma~\ref{lm:equal}, without loss of generality} we can assume that $\rho_t$ has compact support, and the entries of $(\bU,\bW)$ sampled from $\rho_t$ for any $t$ is bounded by $2\|f\|_{\tilde{\cD}_2}$. Next we proceed to control the three terms on the right hand side of~\eqref{eq:direct_relu_pf1}. The following two lemmas give the bounds for the first and third terms.

\begin{lemma}\label{lm:direct_relu1}
$\|f-f^\varepsilon\|^2\leq 4m^2\varepsilon^2\|f\|_{\tilde{\cD}_2}^4$.
\end{lemma}

\noindent{\bf Proof of Lemma~\ref{lm:direct_relu1}}

Let $\bz(t)$ be defined by~\eqref{eq:def_zx} for fixed $\bx$, and $\bz^\varepsilon(t)$ be the solution of the same ODE after replacing $\sigma$ by $\sigma^\varepsilon$. Then, we have $\bz(0)-\bz^\varepsilon(0)=0$, and 
\begin{eqnarray}
|\bz(t)-\bz^\varepsilon(t)| &\leq& \int_0^t \left|\frac{d}{dt}(\bz(s)-\bz^\varepsilon(s))\right|ds \nonumber \\
  &=& \int_0^t \left|\EE_{\rho_s} \bU\sigma(\bW\bz(s))-\EE_{\rho_s}\bU\sigma^\varepsilon(\bW\bz^\varepsilon(s))\right|ds \nonumber \\
  &\leq& \int_0^t \left|\EE_{\rho_s} \bU\sigma(\bW\bz(s))-\EE_{\rho_s}\bU\sigma(\bW\bz^\varepsilon(s))\right|ds \nonumber \\
  && +\int_0^t \left|\EE_{\rho_s} \bU\sigma(\bW\bz^\varepsilon(s))-\EE_{\rho_s}\bU\sigma^\varepsilon(\bW\bz^\varepsilon(s))\right|ds \nonumber \\
  &\leq& \int_0^t \left( \EE_{\rho_s}|\bU||\bW||\bz(s)-\bz^\varepsilon(s)|+2\|f\|_{\tilde{\cD}_2}m\varepsilon \right)ds. \nonumber
\end{eqnarray}
Hence, we have
\begin{equation}
|\bz(1)-\bz^\varepsilon(1)|\leq 2\|f\|_{\tilde{\cD}_2}m\varepsilon N_1(1)e, \nonumber
\end{equation}
where $e$ is an all-one vector. This gives that 
\begin{equation}
    \|f-f^\varepsilon\|^2\leq \int_{D_0} \left(|\alpha|^T|\bz(\bx,1)-\bz^\varepsilon(\bx,1)|\right)^2 d\rho(\bx) \leq 4m^2\varepsilon^2\|f\|_{\tilde{\cD}_2}^4. \nonumber
\end{equation}
\qed

\begin{lemma}\label{lm:direct_relu2}
\begin{equation}
    \EE\|f_L-f^\varepsilon_L\|^2\leq 4m^2\varepsilon^2\|f\|_{\tilde{\cD}_2}^4, \nonumber
\end{equation}
where the expectation is taken over the random choice of weights $\{(\bU_l,\bW_l)\}$.
\end{lemma}

\noindent{\bf Proof of Lemma~\ref{lm:direct_relu2}}

Let $\bz_{l,L}$ be defined by~\eqref{eq:resnet} for a fixed $\bx$, and $\bz^\varepsilon_{l,L}$ be defined similarly with $\sigma$ replaced by $\sigma^\varepsilon$. Then, we have $\bz_{0,L}-\bz^\varepsilon_{0,L}=0$, and 
\begin{eqnarray}
\bz_{l+1,L}-\bz^\varepsilon_{l+1,L}&=&\bz_{l,L}-\bz^\varepsilon_{l,L}+\frac{1}{L}\left[ \bU_l\sigma(\bW_l\bz_{l,L})-\bU_l\sigma(\bW_l\bz^\varepsilon_{l,L}) \right] \nonumber \\
  && +\frac{1}{L}\left[ \bU_l\sigma(\bW_l\bz^\varepsilon_{l,L})-\bU_l\sigma^\varepsilon(\bW_l\bz^\varepsilon_{l,L}) \right]. \nonumber
\end{eqnarray}
Taking absolute value gives
\begin{equation}
|\bz_{l+1,L}-\bz^\varepsilon_{l+1,L}|\leq \left(I+\frac{1}{L}|\bU||\bW|\right)|\bz_{l,L}-\bz^\varepsilon_{l,L}|+\frac{2\|f\|_{\tilde{\cD}_2}m\varepsilon}{L}e, \nonumber
\end{equation}
which implies that
\begin{equation}
|\bz_{L,L}-\bz^\varepsilon_{L,L}|\leq 2\|f\|_{\tilde{\cD}_2}m\varepsilon\prod\limits_{l=0}^{L-1}\left(I+\frac{1}{L}|\bU||\bW|\right)e. \nonumber
\end{equation}
By Theorem~\ref{thm:lln2}, we have
\begin{eqnarray}
 \EE|f_L(\bx)-f_L^\varepsilon(\bx)|^2 &\leq& 4\|f\|_{\tilde{\cD}_2}^2m^2\varepsilon^2 \EE \left(|\alpha|^T\prod\limits_{l=0}^{L-1}\left(I+\frac{1}{L}|\bU||\bW|\right)e\right)^2 \nonumber \\
 &\leq& 4m^2\varepsilon^2\|f\|_{\tilde{\cD}_2}^4. \label{eq:lm_direct_relu2_pf1}
\end{eqnarray}
Integrating~\eqref{eq:lm_direct_relu2_pf1} over $\bx$ gives the results.
\qed

\vspace{2mm}
\noindent{\bf Proof of Theorem~\ref{thm:direct_comp} (Continued)}

With Lemmas~\ref{lm:direct_relu1} and~\ref{lm:direct_relu2}, we have
\begin{equation}\label{eq:expect_decomp}
\EE\|f-f_L\|^2\leq 24m^2\varepsilon^2\|f\|_{\tilde{\cD}_2}^4+3\EE\|f^\varepsilon-f^\varepsilon_L\|^2.
\end{equation}
To bound $\EE\|f^\varepsilon-f^\varepsilon_L\|^2$, let $\be_{l,L}=\sqrt{L}(\bz^\varepsilon_{l,L}-\bz^\varepsilon_{t_{l,L}})$, and recall that we can write
\begin{equation}\label{eq:direct_relu_e_decomp}
\be_{l+1,L}=\be_{l,L}+I_{l,L}+J_{l,L}+K_{l,L},
\end{equation}
with
\begin{eqnarray}
I_{l,L} &=& \frac{1}{\sqrt{L}}\left[ \bU_l\sigma^\varepsilon(\bW_l\bz^\varepsilon_{l,L})-\bU_l\sigma^\varepsilon(\bW_l\bz^\varepsilon(t_{l,L})) \right], \nonumber \\
J_{l,L} &=& \frac{1}{\sqrt{L}}\left[ \bU_l\sigma^\varepsilon(\bW_l\bz^\varepsilon(t_{l,L}))-\EE_{\rho_{t_{l,L}}}\bU\sigma^\varepsilon(\bW\bz^\varepsilon(t_{l,L})) \right]\nonumber \\
K_{l,L} &=& \frac{1}{\sqrt{L}}\left[ \EE_{\rho_{t_{l,L}}}\bU\sigma^\varepsilon(\bW\bz^\varepsilon(t_{l,L}))-L\int_{t_{l,L}}^{t_{l+1,L}}\EE_{\rho_{t}}\bU\sigma^\varepsilon(\bW\bz^\varepsilon(t))dt  \right]. \nonumber
\end{eqnarray}

For $I_{l,L}$, by the Taylor expansion of $\bU_l\sigma^\varepsilon(\bW_l\bz^\varepsilon_{l,L})$ at $\bz^\varepsilon(t_{l,L})$, we get
\begin{eqnarray}
I_{l,L}&=&\frac{1}{L}\bU_l(\sigma^\varepsilon(\bW_l\bz^\varepsilon(t_{l,L})))'\bW_l\be_{l,L}+\frac{\bU_l(\sigma^\varepsilon(\bW_l\bz^\varepsilon(t_{l,L})))''{ (\bW_l\be_{l,L})\circ(\bW_l\be_{l,L})}}{L\sqrt{L}}, \label{eq:direct_relu_pf2}
\end{eqnarray}
{ where for two vectors $\alpha$ and $\beta$, $\alpha\circ\beta$ means element-wise product.} For the second term on the right hand side of~\eqref{eq:direct_relu_pf2}, we have
\begin{equation}
\left|\bU_l(\sigma^\varepsilon(\bW_l\bz^\varepsilon(t_{l,L})))''(\bW_l\be_{l,L})\circ(\bW_l\be_{l,L})\right|\leq\frac{8\|f\|_{\tilde{\cD}_2}^3mD\|\be_{l,L}\|^2}{\varepsilon}e. \nonumber
\end{equation}
On the other hand, for $K_{l,L}$ we have
\begin{equation}
|K_{l,L}|\leq\frac{C_2}{L\sqrt{L}}e, \nonumber
\end{equation}
for some constant $C_2$. Hence, we can write~\eqref{eq:direct_relu_e_decomp} as
\begin{equation}\label{eq:direct_relu_pf3}
\be_{l+1,L}=\be_{l,L}+\frac{1}{L}\bU_l(\sigma^\varepsilon(\bW_l\bz^\varepsilon(t_{l,L})))'\bW_l\be_{l,L}+J_{l,L}+\frac{r_{l,L}}{L\sqrt{L}},
\end{equation}
with 
\begin{equation}
|r_{l,L}|\leq (8\|f\|_{\tilde{\cD}_2}^3mD\varepsilon^{-1}\|\be_{l,L}\|^2+C_2)e. \nonumber
\end{equation}

Next, we consider $\be_{l,L}\be_{l,L}^T$. By~\eqref{eq:direct_relu_pf3}, we have
\begin{eqnarray}
\be_{l+1,L}\be_{l+1,L}^T &=& \be_{l,L}\be_{l,L}^T+\frac{1}{L}\left(\bU_l(\sigma^\varepsilon(\bW_l\bz^\varepsilon(t_{l,L})))'\bW_l\be_{l,L}\be_{l,L}^T+\be_{l,L}\be_{l,L}^T \bU_l(\sigma^\varepsilon(\bW_l\bz^\varepsilon(t_{l,L})))'\bW_l\right) \nonumber \\
  && +J_{l,L}J_{l,L}^T+\frac{1}{L^2}\bU_l(\sigma^\varepsilon(\bW_l\bz^\varepsilon(t_{l,L})))'\bW_l\be_{l,L}(\bU_l(\sigma^\varepsilon(\bW_l\bz^\varepsilon(t_{l,L})))'\bW_l\be_{l,L})^T \nonumber \\
  && + \be_{l,L}J_{l,L}^T+J_{l,L}\be_{l,L}^T+\frac{1}{L\sqrt{L}}\left(\be_{l,L}r_{l,L}^T+r_{l,L}\be_{l,L}\right)+\frac{1}{L^3}r_{l,L}r_{l,L}^T \nonumber \\
  && +\frac{1}{L}\bU_l(\sigma^\varepsilon(\bW_l\bz^\varepsilon(t_{l,L})))'\bW_l\be_{l,L}J_{l,L}^T+\frac{1}{L}J_{l,L}(\bU_l(\sigma^\varepsilon(\bW_l\bz^\varepsilon(t_{l,L})))'\bW_l\be_{l,L})^T \nonumber \\
  && +\frac{1}{L^2\sqrt{L}}\bU_l(\sigma^\varepsilon(\bW_l\bz^\varepsilon(t_{l,L})))'\bW_l\be_{l,L}r_{l,L}^T+\frac{1}{L^2\sqrt{L}}r^{l,L}(\bU_l(\sigma^\varepsilon(\bW_l\bz^\varepsilon(t_{l,L})))'\bW_l\be_{l,L})^T \nonumber \\
  && +\frac{1}{L\sqrt{L}}\left(J_{l,L}r_{l,L}^T+r_{l,L}J_{l,L}^T\right). \nonumber
\end{eqnarray}
Taking expectation over the equation above, noting that $J_{l,L}$ is independent with $\be_{l,L}$, and using the bound of $r_{l,L}$ we derived above, we get
\begin{eqnarray}
|\EE\be_{l+1,L}\be_{l+1,L}^T| &\leq& |\EE\be_{l,L}\be_{l,L}^T|+\frac{1}{L}\left(A_{l,L}|\EE\be_{l,L}\be_{l,L}^T|+|\EE\be_{l,L}\be_{l,L}^T|A_{l,L}^T\right)+\frac{1}{L}\Sigma_{l,L} \nonumber \\
  && + \frac{C\|f\|_{\tilde{\cD}_2}^3}{L}\left(\frac{mD\EE\|\be_{l,L}\|^3}{\sqrt{L}\varepsilon}+\frac{m^2D^2\EE\|\be_{l,L}\|^4}{L^2\varepsilon^2}\right)E, \label{eq:direct_relu_pf4}
\end{eqnarray}
where
$$A_{l,L}=\EE_{\rho_{t_{l,L}}}|\bU||\bW|,\ \ \Sigma_{l,L}=\left|\textrm{Cov}_{\rho_{t_{l,L}}}\bU_l\sigma^\varepsilon(\bW_l\bz^\varepsilon(t_{l,L}))\right|,$$
$E$ is an all-one matrix and $C$ is a constant. 

Next, we bound the third and fourth order moments of $\|\be_{l,L}\|$ using its second order moment. This is done by the following lemma.

\begin{lemma}\label{lm:direct_relu3}
For any $L$ and $1\leq l\leq L$, there exists a constant $C$ such that
\begin{equation}
\EE\|\be_{l,L}\|^3 \leq CmD^{3/2}\|f\|_{\tilde{\cD}_2}\left(\sqrt{\log L}+\frac{D}{\sqrt{L}\varepsilon}\right)\EE\|\be_{l,L}\|^2, \nonumber
\end{equation}
and
\begin{equation}
\EE\|\be_{l,L}\|^4\leq C^2m^2D^{3}\|f\|_{\tilde{\cD}_2}^2\left(\sqrt{\log L}+\frac{D}{\sqrt{L}\varepsilon}\right)^2\EE\|\be_{l,L}\|^2. \nonumber
\end{equation}
\end{lemma}

\noindent{\bf Proof of Lemma~\ref{lm:direct_relu3}}

Let $S_{l,L}=\sum\limits_{k=0}^{l-1}J_{k,L}$. Then, $\EE S_{l,L}=0$. Since $J_{l,L}$ are independent for different $l$, and
\begin{equation}
|J_{l,L}|\leq\frac{C'mD}{\sqrt{L}}e \nonumber
\end{equation}
holds for all $l$ and some constant $C'$, by Hoeffding's inequality, for any $t>0$ and $1\leq i\leq D$ we have
\begin{equation}
\PP(|S_{l,L,i}|\geq t)\leq2\exp(-\frac{t^2}{2C'^2m^2D^2}). \nonumber
\end{equation}
Here, $S_{l,L,i}$ denotes the $i$-th entry of the vector $S_{l,L}$. Taking $t=2C'mD\sqrt{\log L}$, we obtain
\begin{equation}
\PP\left(|S_{l,L,i}|\geq 2C'mD\sqrt{\log L}\right)\leq\frac{2}{L^2}. \nonumber
\end{equation}
This implies
\begin{eqnarray}
\PP\left(\|S_{l,L}\|\geq 2C'mD^{3/2}\sqrt{\log L}\right) &=& 1-\PP\left(\|S_{l,L}\|< 2C'mD^{3/2}\sqrt{\log L}\right) \nonumber \\
  &\leq & 1-\PP\left(\bigcup_{i}\left\{|S_{l,L,i}|< 2C'mD\sqrt{\log L}\right\}\right) \nonumber \\
  &\leq & 1-\left(1-\frac{2}{L^2}\right)^D \nonumber \\
  &\leq & \frac{2D}{L^2}. \label{eq:lm_direct_relu3_pf1}
\end{eqnarray}
Define the event $\cA$ by
\begin{equation}
\cA=\left\{ \|S_{l,L}\|\leq 2C'mD^{3/2}\sqrt{\log L},\ i=1,2,\cdots,L \right\}. \nonumber
\end{equation}
Then by~\eqref{eq:lm_direct_relu3_pf1} we have
\begin{equation}
\PP\left( \cA \right)\geq 1-\frac{2D}{L}. \nonumber
\end{equation}

Using~\eqref{eq:direct_relu_pf3}, we have
\begin{equation}
\be_{l,L}= \sum\limits_{k=0}^{l-1}\frac{1}{L}\bU_k(\sigma^\varepsilon(\bW_k\bz^\varepsilon(t_{k,L})))'\bW_k\be_{k,L}+S_{l,L}+\sum\limits_{k=0}^{l-1}\frac{r_{k,L}}{L\sqrt{L}}. \nonumber
\end{equation}
Hence, using the bounds of $S_{l,L}$ and $r_{k,L}$, we obtain that there is a constant $C$ such that
\begin{equation}
    \|\be_{l,L}\|\leq CmD^{3/2}\|f\|_{\tilde{\cD}_2}\left(\sqrt{L}+\frac{1}{\sqrt{L}\varepsilon}\right). \nonumber
\end{equation}
On the other hand, under event $\cA$, using the sharper bound of $S_{l,L}$, we have
\begin{equation}
    \|\be_{l,L}\|\leq CmD^{3/2}\|f\|_{\tilde{\cD}_2}\left(\sqrt{\log L}+\frac{1}{\sqrt{L}\varepsilon}\right). \nonumber
\end{equation}
For third-order moment of $\|\be_{l,L}\|$, we have
\begin{eqnarray}
\EE\|\be_{l,L}\|^3 &\leq& CmD^{3/2}\|f\|_{\tilde{\cD}_2}\left( \left(\sqrt{\log L}+\frac{1}{\sqrt{L}\varepsilon}\right)\PP(\cA) + \left(\sqrt{L}+\frac{1}{\sqrt{L}\varepsilon}\right)\PP(\cA^c) \right)\EE\|\be_{l,L}\|^2 \nonumber \\
  &\leq& CmD^{3/2}\|f\|_{\tilde{\cD}_2}\left(\sqrt{\log L}+\frac{D}{\sqrt{L}\varepsilon}\right)\EE\|\be_{l,L}\|^2. \nonumber
\end{eqnarray}
Similarly, for fourth order moment we have
\begin{equation}
\EE\|\be_{l,L}\|^4\leq C^2m^2D^{3}\|f\|_{\tilde{\cD}_2}^2\left(\sqrt{\log L}+\frac{D}{\sqrt{L}\varepsilon}\right)^2\EE\|\be_{l,L}\|^2. \nonumber
\end{equation}
\qed

\noindent{\bf Proof of Theorem~\ref{thm:direct_comp} (Continued)}

Applying the results of Lemma~\ref{lm:direct_relu3} to~\eqref{eq:direct_relu_pf4} gives
\begin{eqnarray}
|\EE\be_{l+1,L}\be_{l+1,L}^T| &\leq& |\EE\be_{l,L}\be_{l,L}^T|+\frac{1}{L}\left(A_{l,L}|\EE\be_{l,L}\be_{l,L}^T|+|\EE\be_{l,L}\be_{l,L}^T|A_{l,L}^T\right)+\frac{1}{L}\Sigma_{l,L} \nonumber \\
  && + \frac{C}{L}\left(m^4D^5\|f\|_{\tilde{\cD}_2}^5\left(\frac{\sqrt{\log L}}{\sqrt{L}\varepsilon}+\frac{D}{L\varepsilon^2}\right)\EE\|\be_{l,L}\|^2\right)E. \label{eq:direct_relu_pf5}
\end{eqnarray}
Since $\|f\|_{\tilde{\cD}_2}<\infty$, $\Sigma_{l,L}$ is uniformly bounded. Without loss of generality, we can assume $\Sigma_{l,L}\leq CE$.
Furthermore, assume $L$ is sufficiently large such that 
\begin{equation}\label{eq:direct_relu_pf6}
\frac{m^4D^6\|f\|_{\tilde{\cD}_2}^5\EE\|\be_{l,L}\|^2}{L^{\delta/3}}\leq1.
\end{equation}
Then, from~\eqref{eq:direct_relu_pf5} we have
\begin{eqnarray}
|\EE\be_{l+1,L}\be_{l+1,L}^T| &\leq& |\EE\be_{l,L}\be_{l,L}^T|+\frac{1}{L}\left(A_{l,L}|\EE\be_{l,L}\be_{l,L}^T|+|\EE\be_{l,L}\be_{l,L}^T|A_{l,L}^T\right) \nonumber \\
  && + \frac{C}{L}\left(1+\frac{\sqrt{\log L}}{L^{1/2-\delta/3}\varepsilon}+\frac{D}{L^{1-\delta/3}\varepsilon^2}\right)E, \nonumber
\end{eqnarray}
which implies that 
\begin{equation}\label{eq:direct_relu_pf7}
|\EE\be_{l+1,L}\be_{l+1,L}^T| \leq C\left(1+\frac{\sqrt{\log L}}{L^{1/2-\delta/3}\varepsilon}+\frac{D}{L^{1-\delta/3}\varepsilon^2}\right) N_1(1) N_1(1)^T,
\end{equation}
and thus
\begin{equation}\label{eq:direct_relu_pf8}
\EE\|\be_{l,L}\|^2\leq e^T|\EE\be_{l+1,L}\be_{l+1,L}^T|e\leq C\left(1+\frac{\sqrt{\log L}}{L^{1/2-\delta/3}\varepsilon}+\frac{D}{L^{1-\delta/3}\varepsilon^2}\right) (e^TN_1(1))^2.
\end{equation}
Note that $e^TN_1(1)=\|N_1(1)\|_1\leq\|f\|_{\tilde{\cD}_2}+D$. By~\eqref{eq:direct_relu_pf6} and~\eqref{eq:direct_relu_pf8}, \eqref{eq:direct_relu_pf7} happens if 
\begin{equation}
Cm^4D^6\|f\|_{\tilde{\cD}_2}^5\left(1+\frac{\sqrt{\log L}}{L^{1/2-\delta/3}\varepsilon}+\frac{D}{L^{1-\delta/3}\varepsilon^2}\right) (\|f\|_{\tilde{\cD}_2}+D)^2\leq L^{\delta/3}. \nonumber
\end{equation} 
Taking $\varepsilon=L^{-1/2+\delta/3}$, it suffices to have
\begin{equation}
Cm^4D^6\|f\|_{\tilde{\cD}_2}^5\left(1+D+\sqrt{\log L}\right) (\|f\|_{\tilde{\cD}_2}+D)^2\leq L^{\delta/3}. \nonumber
\end{equation} 

In this case, we have
\begin{equation}
\EE \|f^\varepsilon-f_L^\varepsilon\|^2\leq \frac{C}{L}\left(1+D+\sqrt{\log L}\right)\|f\|_{\tilde{\cD}_2}^2. \nonumber
\end{equation}
Plugging into~\eqref{eq:expect_decomp} gives
\begin{equation}
\EE\|f-f_L\|^2\leq \frac{24m^2}{L^{1-2\delta/3}}\|f\|_{\tilde{\cD}_2}^4+\frac{3C}{L}\left(1+D+\sqrt{\log L}\right)\|f\|_{\tilde{\cD}_2}^2. \nonumber
\end{equation}
When $L$ sufficiently large (larger than polynomial of $m$, $D$, $\log L$), we have
\begin{equation}
\EE\|f-f_L\|^2\leq \frac{\|f\|_{\tilde{\cD}_2}^2}{3L^{1-\delta}}. \nonumber
\end{equation}

Note that the bound above holds for any fixed $\bx\in \bX$. Now, integrating over $\bx$, we have
\begin{equation}
\EE \|f-f_L\|^2=\int \EE|f(\bx)-f_L(\bx)|^2d\mu(\bx)\leq\frac{\|f\|_{\tilde{\cD}_2}^2}{3L^{1-\delta}}. \nonumber
\end{equation}
By Markov's inequality, with probability no less than $\frac{2}{3}$, the distance between $f$ and $f_L$ can be controlled by
\begin{equation}\label{eq:approx_pf1}
\|f-f_L\|^2\leq\frac{\|f\|_{\tilde{\cD}_2}^2}{L^{1-\delta}}.
\end{equation}

Next, consider the path norm of $f_L$, which is defined as 
\begin{equation}
\|f_L\|_{P}=\left\||\alpha|\prod\limits_{l=1}^L\left(I+\frac{1}{L}|\bU_l||\bW_l|\right)|\bV|\right\|_1. \nonumber
\end{equation}
Define a recurrent scheme,
\begin{eqnarray}
\by_{0,L} &=& \bV, \nonumber \\
\by_{l+1,L} &=& \by_{l,L}+\frac{1}{L}|\bU_l||\bW_l|\by_{l,L}. \nonumber
\end{eqnarray}
Using Theorem~\ref{thm:lln2} with $\sigma$ being the identity function and $\bU$ and $\bW$ replaced by $|\bU|$ and $|\bW|$
 respectively, we know that $\||\alpha|^T\by_{L,L}\|_1\rightarrow\|f\|_{\cD_1(\rho_t)}$ almost surely. Hence, by taking $\rho_t$ such that $\|f\|_{\cD_1(\rho_t)}\leq2\|f\|_{\cD_1}$, we have
\begin{equation}
\EE \|f_L\|_P\leq3\|f\|_{\cD_1}, \nonumber
\end{equation}
when $L$ is sufficiently large. Again using Markov's inequality, with probability no less than $\frac{2}{3}$, we have 
\begin{equation}\label{eq:approx_pf2}
\EE \|f_L\|_P\leq9\|f\|_{\cD_1}.
\end{equation}

Combining the result above with~\eqref{eq:approx_pf1}, we know that with probability no less than $\frac{1}{3}$, we have both~\eqref{eq:approx_pf1} and~\eqref{eq:approx_pf2}. Therefore, we can find an $f_L$ that satisfies both~\eqref{eq:approx_pf1} and~\eqref{eq:approx_pf2}.  This completes the proof.
\qed
\ \\

\subsubsection{ Proof of Theorem~\ref{thm:inverse}}

For any $L$, let $f_L(\cdot)$ be the residual network represented by the parameters $\balpha_L$, $\{\bU^L_l, \bW^L_l\}_{l=0}^{L-1}$ and $\bV$. Let $\bz_{l,L}(\bx)$ be the function represented by the $l$-th layer of network $f_L$, then $f_L(\bx)=\balpha_L^T\bz_{L,L}(\bx)$  for all $\bx \in \bX$.
Since $\balpha_L$ uniformly bounded for all $L$, there exists a subsequence $L_k$ and $\balpha$ such that 
\begin{equation}
  \balpha_{L_k}\rightarrow\balpha, \nonumber
\end{equation}
when $k\rightarrow\infty$. Without loss of generality, we assume $\balpha_L\rightarrow\balpha$. 

Let $\bU_t^L: [0,1]\rightarrow\bR^{D\times m}$ 
be a piecewise constant function defined by
\begin{equation}
    \bU_t^L = \bU_l^L,\ for\ t\in[\frac{l}{L},\frac{l+1}{L}), \nonumber
\end{equation}
and $\bU_1^L=\bU_{L-1}^L$. Similarly we can define $\bW_t^L$. Then, $\{\bU_t^L\}$ and $\{\bW_t^L\}$ are uniformly bounded. Hence, by the fundamental theorem for Young measures \cite{young2000lecture, ball1989version}, there exists a subsequence $\{L_k\}$ and a family of probability measure $\{\rho_t, t \in [0, 1]\}$, such that for every Caratheodory function $F$, 
\begin{equation}
    \lim_{k\rightarrow\infty} \int_0^1 F(\bU^{L_k}_t,\bW^{L_k}_t,t)dt = \int_0^1 \EE_{\rho_t} F(\bU,\bW,t)dt. \nonumber
\end{equation}
Let $\tilde{f}=f_{\balpha,\{\rho_t\}}$.  We are going to show $\tilde{f}=f$.
Let $\bz_Y(\cdot,t)$ be defined by $\bz_Y(\bx,0)=\bV\bx$ and
\begin{equation}
    \bz_Y(\bx,t)=\bz_Y(\bx,0)+\int_0^t \EE_{\rho_t} U\sigma(W\bz_Y(\bx,s))ds. \nonumber
\end{equation}
Then it suffices to show that 
\begin{equation}\label{eq:conv_z}
\lim_{k\rightarrow\infty} \bz_{L_k,L_k}(\bx)\rightarrow\bz_Y(\bx,1),
\end{equation}
for any fixed $\bx \in D_0$.

To prove~\eqref{eq:conv_z}, we first consider the following continuous version of $\bz_{l,L}$,
\begin{eqnarray}
\bz_L(\bx,0) &=& \bz_{0,L}(\bx), \nonumber \\
\frac{d}{dt}\bz_L(\bx,t) &=& \bU_t^L\sigma(\bW_t^L\bz_L(\bx,t)), \nonumber
\end{eqnarray}
and show that $|\bz_L(\bx,1)-\bz_{L,L}(\bx)|\rightarrow0$. To see this, note that
\begin{eqnarray}
\bz_L(\bx,t_{l+1,L}) &=& \bz_L(\bx,t_{l,L})+\int_{t_{l,L}}^{t_{l+1,L}} \bU_t^L\sigma(\bW_t^L\bz_L(\bx,s))ds, \label{eq:z_cont} \\
\bz_{l+1,L}(\bx) &=& \bz_{l,L}(\bx)+\int_{t_{l,L}}^{t_{l+1,L}} \bU_t^L\sigma(\bW_t^L\bz_{l,L}(\bx))ds. \label{eq:z_dis}
\end{eqnarray}
Subtracting~\eqref{eq:z_cont} from~\eqref{eq:z_dis}, and let $\be_{l,L}=\bz_{l,L}(\bx)-\bz_L(\bx,t_{l,L})$, we have
\begin{eqnarray}
\be_{l+1,L} &=& \be_{l,L}+\int_{t_{l,L}}^{t_{l+1,L}} \left(\bU_t^L\sigma(\bW_t^L\bz_{l,L}(\bx))-\bU_t^L\sigma(\bW_t^L\bz_L(\bx,s))\right)ds \nonumber \\
  &=& \be_{l,L}+\int_{t_{l,L}}^{t_{l+1,L}} \left(\bU_t^L\sigma(\bW_t^L\bz_{l,L}(\bx))-\bU_t^L\sigma(\bW_t^L\bz_L(\bx,t_{l,L}))\right)ds \nonumber \\
  && + \int_{t_{l,L}}^{t_{l+1,L}} \left(\bU_t^L\sigma(\bW_t^L\bz_L(\bx,t_{l,L}))-\bU_t^L\sigma(\bW_t^L\bz_L(\bx,s))\right)ds. \label{eq:young_e}
\end{eqnarray}
Since $\{\bU_t^L\}$ and $\{\bW_t^L\}$ are bounded, we know that $\{\bz_L(\bx,t)\}$ is bounded, and $\{\frac{d}{dt}\bz_L(\bx,t)\}$ is also bounded. Hence, there exists a uniform constant $C$ such that
\begin{eqnarray}
  \left\|\bU_t^L\sigma(\bW_t^L\bz_{l,L}(\bx))-\bU_t^L\sigma(\bW_t^L\bz_L(\bx,t_{l,L}))\right\| &\leq& C\|e_{l,L}\|, \label{eq:young_est1} \\
  \left\|\bU_t^L\sigma(\bW_t^L\bz_L(\bx,t_{l,L}))-\bU_t^L\sigma(\bW_t^L\bz_L(\bx,s))\right\| &\leq& C|s-t_{l,L}|. \label{eq:young_est2}.
\end{eqnarray}
Plugging~\eqref{eq:young_est1} and~\eqref{eq:young_est2} into~\eqref{eq:young_e}, we obtain
\begin{equation}
\|\be_{l+1,L}\|\leq \left(1+\frac{C}{L}\right)\|\be_{l,L}\|+\frac{C}{L^2}. \nonumber
\end{equation}
Therefore, by Gronwall's inequality, $\|\be_{L,L}\|\leq\cO(1/L)$, which gives
\begin{equation}\label{eq:conv_cont_dis}
|\bz_L(\bx,1)-\bz_{L,L}(\bx)|\rightarrow0.
\end{equation}

Now with~\eqref{eq:conv_cont_dis}, we only need to show 
\begin{equation}\label{eq:conv_z2}
\lim_{k\rightarrow\infty} \bz_{L_k}(\bx,1)\rightarrow\bz_Y(\bx,1), \nonumber
\end{equation}
which is equivalent to showing that for any $\epsilon$, there exists $K>0$ such that for any $k>K$, we have
\begin{equation}
    \|\bz_{L_k}(\bx,1)-\bz_Y(\bx,1)\|\leq\epsilon. \nonumber
\end{equation}

For a large integer $N$, let $t_{i,N}=i/N$. By the definition of $\bz_Y$ and $\bz_{L_k}$, we have
\begin{equation}
\bz_{L_k}(\bx,t_{i+1,N})=\bz_{L_k}(\bx,t_{i,N})+\int_{t_{i,N}}^{t_{i+1,N}}\bU_s^{L_k}\sigma(\bW_s^{L_k}\bz_{L_k}(\bx,s))ds, \nonumber
\end{equation}
and 
\begin{equation}
\bz_Y(\bx,t_{i+1,N})=\bz_Y(\bx,t_{i,N})+\int_{t_{i,N}}^{t_{i+1,N}} \EE_{\rho_t} \bU\sigma(\bW \bz_Y(\bx,s))ds. \nonumber
\end{equation}
Let $r_{i,N}(\bx)=\bz_Y(\bx,t_{i,N})-\bz_{L_k}(\bx,t_{i,N})$, and note that $\{\bU_t^{L_k}\}$ and $\{\bW_t^{L_k}\}$ are bounded, we have
\begin{eqnarray}
\|r_{i+1,N}(\bx)\| &\leq& \left(1+\frac{C}{N}\right)\|r_{i,N}\|+\frac{C}{N^2} \nonumber \\
  && +\left\|\int_{t_{i,N}}^{t_{i+1,N}}\left[  \bU_s^{L_k}\sigma(\bW_s^{L_k}\bz_Y(\bx,s))-\EE_{\rho_t} \bU\sigma(\bW \bz_Y(\bx,s))\right]ds \right\|, \nonumber
\end{eqnarray}
for some constant $C$.
Using the theorem for  Young measures \cite{young2000lecture, ball1989version}, there exists a sufficiently large $K$, such that for all $k>K$, we have
\begin{equation}
\left\|\int_{t_{i,N}}^{t_{i+1,N}}\left[  \bU_s^{L_k}\sigma(\bW_s^{L_k}\bz_Y(\bx,s))-\EE_{\rho_t} \bU\sigma(\bW \bz_Y(\bx,s))\right]ds \right\| \leq \frac{1}{N^2}, \nonumber
\end{equation}
for all $0\leq i\leq N-1$.  By Gronwall's inequality, there exists a constant $\tilde{C}$ such that
\begin{equation}
\|r_{N,N}(\bx)\|\leq \frac{\tilde{C}}{N}. \nonumber
\end{equation}
If we take $N=\epsilon/\tilde{C}$, we have
\begin{equation}
    \|\bz_{L_k}(\bx,1)-\bz_Y(\bx,1)\|\leq\epsilon, \nonumber
\end{equation}
for sufficiently large $k$. This shows that $f=f_{\balpha,\{\rho_t\}}$. 

To bound the $\cD_\infty$ norm of $f$, take $F$ as the indicator function of $\{|\bU|\leq c_0, |\bW|\leq c_0\}^c$ and apply the theorem for Young measures, we obtain that for any $t\in [0, 1]$, the support of $\rho_t$ lies in $\{|\bU|\leq c_0, |\bW|\leq c_0\}$. Hence, $f\in\cD_\infty$. To estimate $\|f\|_{\cD_\infty}$, consider $N_\infty(t)$ defined by~\eqref{eq:n_ode}, since the
elements of $\bU$ and $\bW$ are bounded by $c_0$, we have
\begin{equation}
    \dot{N}_\infty(t)\leq mc_0^2 EN_\infty(t), \nonumber
\end{equation}
where $E$ is an all-one $D\times D$ matrix. Therefore, we have
\begin{equation}
N_\infty(1)\leq e^{mc_0^2E}\be \leq\frac{2De^{m(c_0^2+1)}}{m}\be. \nonumber
\end{equation}
Since the elements of $\balpha$ are also bounded by $c_0$, we get
\begin{equation}
\|f\|_{\cD_\infty}\leq |\balpha|^TN_\infty(1)\leq \frac{2D^2e^{m(c_0^2+1)}c_0}{m}. \nonumber
\end{equation}

Finally, if $\|f_L\|_{\cD_1}\leq c_1$ holds for all $L>0$, then using 
the technique of treating $\bz_Y(\bx,t)$ on $N_1(t)$, we 
obtain $\|f\|_{\cD_1}\leq c_1$.

\qed
\ \\

\subsubsection{ Proof of Theorem~\ref{thm:rad_res}}
Similar to the proof of Theorem~\ref{thm:direct_comp}, we can define a discrete analogy of the $\hat{\cD}_1$ norm for residual network
\begin{equation}
\|\Theta\|_{\textrm{WP}}=|\balpha|^T\prod\limits_{l=1}^L\left(I+\frac{2}{L}|\bU_l||\bW_l|\right)\be. \nonumber
\end{equation} 
Using the same techniques as for the direct approximation theorem, we can show that any functions in $\hat{\cD}_2^Q$ can be approximated by a series of residual networks $f_L(\cdot;\Theta_L)$ with depth $L$ tends to infinity and $\|\Theta_L\|_{\textrm{WP}}\leq 9Q$. Here we use WP (weighted path) to denote the discrete norm because this norm is a weighted version of the original path norm, and assigns larger weights for those paths going through more non-linearities. 
Let $\cF^Q$ be the set of all residual networks whose weighted path norms are bounded by $Q$, i.e.
$$
\cF^Q = \{f(\cdot;\Theta):\ f(\cdot;\Theta) \textrm{\ is\ a\ residual\ network\ and\ } \|\Theta\|_{\textrm{WP}}\leq Q\},
$$
and let $\overline{\cF}^Q$ be the closure of $\cF^Q$. Then, by the direct approximation results, $\hat{\cD}_2^Q\subset\hat{\cD}_1^Q\subset\overline{\cF}^Q$. Hence, $\rad_n(\hat{\cD}_2^Q)\leq\rad_n(\overline{\cF}^{9Q})$. On the other hand, in~\cite{ma2019priori} it is proven that
$$\rad_n(\overline{\cF}^Q)\leq 2Q\sqrt{\frac{2\log(2d)}{n}}. $$ Therefore, 
\begin{equation}
\rad_n(\hat{\cD}_2^Q)\leq 18Q\sqrt{\frac{2\log(2d)}{n}} \nonumber
\end{equation}

\qed
\ \\

\section{Concluding remarks}

As far as the high dimensional approximation theory is concerned, we are interested in approximation schemes (or machine learning models)
 that satisfy
$$
\|f - f_m \|^2 \le C_0 \frac{\gamma(f)^2}{m}
$$
for $f$ is a certain function space $\mathcal{F}$ defined by the particular approximation scheme or machine learning model.
Here $\gamma$ is a functional defined on $\mathcal{F}$, typically a norm for the function space.
It plays the role of the variance in the context  of Monte Carlo integration.
A machine learning model is preferred if its associated function space $\mathcal{F}$ is large and the functional $\gamma$ is small.

However, practical machine learning models can only work with a finite dataset on which the values of the target function are known.
This results in an additional error, the estimation error, in the total error of the machine learning model.
The estimation error is controlled by the Rademacher complexity of the hypothesis  space, which can be thought of as a
truncated version of the space $\mathcal{F}$. It just so happens that for the spaces identified here the Rademacher complexity has the optimal estimates:
$$ 
\rad_n (\mathcal{F}_Q) \le C_0 \frac{Q}{\sqrt{n}}.
$$
 This is also true for the RKHS.
It is not clear whether this is a coincidence, or there are some more fundamental reasons behind.

Whatever the reason, the combination of these two results imply that the
generalization error (also called population risk) should have the optimal scaling
$ O(1/m) + O(1/\sqrt{n})$ for all three methods: the kernel method, the two-layer
neural networks and residual networks. The difference lies in the 
coefficients hidden in the above expression.
These coefficients are the norms of the target function
in the corresponding function spaces.
In this sense, going from the kernel method to two-layer neural networks and to deep residual neural networks is like a variance reduction process
since the value of the norms decreases in this process.
In addition, the function space $\mathcal{F}$ expands substantially from some RKHS to the Barron space and to the flow-induced function space.

{\bf Acknowledgement:} The work presented here is supported in part by
a gift to Princeton University from iFlytek and the ONR grant N00014-13-1-0338.

\bibliographystyle{plain}
\bibliography{dl_ref}

\end{document}